%% file: limart_se3.tex
\documentclass[runningheads]{llncs}

\usepackage[rightcaption]{sidecap}
\sidecaptionvpos{figure}{c}
\sidecaptionvpos{table}{c}

\usepackage{graphicx}
\graphicspath{{../}{./}}
\usepackage{comment}

\usepackage{amsmath,amssymb,amsthm} % define this before the line numbering.
\usepackage{color}

\usepackage[width=122mm,left=12mm,paperwidth=146mm,height=193mm,top=12mm,paperheight=217mm]{geometry}

\usepackage{times}

\usepackage{multirow}

\usepackage{enumitem}
\setlist[enumerate]{topsep=1ex,itemsep=0ex,partopsep=0ex,parsep=1ex}
\setlist[itemize]{topsep=1ex,itemsep=0ex,partopsep=0ex,parsep=1ex}

\usepackage[strings]{underscore}

\usepackage{caption}

\usepackage[referable]{threeparttablex}

\usepackage{appendix}

\DeclareMathOperator{\tr}{tr}
\newcommand{\SE}{\mathrm{SE}}

\theoremstyle{definition}
\newcommand{\PLH}{{\mkern-2mu\times\mkern-2mu}}

\begin{document}

\pagestyle{headings}
\mainmatter
\def\ECCVSubNumber{770}  % Insert your submission number here

\title{Relative Pose Estimation of Calibrated Cameras with Known $\SE(3)$ Invariants}

\begin{comment}
\titlerunning{ECCV-20 submission ID \ECCVSubNumber} 
\authorrunning{ECCV-20 submission ID \ECCVSubNumber} 
\author{Anonymous ECCV submission}
\institute{Paper ID \ECCVSubNumber}

\end{comment}
%******************

% \begin{comment}
% \titlerunning{Relative Pose Estimation of Calibrated Cameras with Known $\SE(3)$ Invariants}
\author{Bo Li\inst{1}\orcidID{0000-0002-9336-1862},
Evgeniy	Martyushev\inst{2}\orcidID{0000-0002-6892-079X}, \\
Gim Hee Lee\inst{1}\orcidID{0000-0002-1583-0475}}
%
% \authorrunning{Bo Li \and Evgeniy Martyushev \and Gim Hee Lee}
\institute{National University of Singapore, Singapore\\
\email{prclibo@gmail.com, gimhee.lee@nus.edu.sg}\and
South Ural State University, 454080 Chelyabinsk, Russia\\
\email{martiushevev@susu.ru}}
% \end{comment}
%******************

\maketitle

\begin{abstract}

    The $\SE(3)$ invariants of a pose include its rotation angle and screw translation. In this paper, we present a complete comprehensive study of the relative pose estimation problem for a calibrated camera constrained by known $\SE(3)$ invariant, which involves 5 minimal problems in total. These problems reduces the minimal number of point pairs for relative pose estimation and improves the estimation efficiency and robustness. The $\SE(3)$ invariant constraints can come from extra sensor measurements or motion assumption. Different from conventional relative pose estimation with extra constraints, no extrinsic calibration is required to transform the constraints to the camera frame. This advantage comes from the invariance of $\SE(3)$ invariants cross different coordinate systems on a rigid body and makes the solvers more convenient and flexible in practical applications.
    Besides proposing the concept of relative pose estimation constrained by $\SE(3)$ invariants, we present a comprehensive study of existing polynomial formulations for relative pose estimation and discover their relationship. Different formulations are carefully chosen for each proposed problems to achieve best efficiency. Experiments on synthetic and real data shows performance improvement compared to conventional relative pose estimation methods.

% The $\SE(3)$ invariants of a pose include its rotation angle and screw translation. In this paper, we study the relative pose estimation problem for a calibrated camera in cases where 1) the rotation angle is known and/or 2) the screw translation is known to be zero. 
% The first case applies to robotic systems with available rotation sensors such as the Inertial Measurement Units (IMUs). The second case applies to robotic systems under planar motion constraint. Involving $\SE(3)$ invariant measurements reduces the minimal number of point pairs for relative pose estimation and improves the numerical accuracy. Moreover, since the $\SE(3)$ invariants do not change under a rigid body transform, there is no requirement on the camera mounting position or extrinsic calibration between the camera and the rotation sensor.
% To this end, we propose two new relative pose estimators based on the $\SE(3)$ invariant measurements. The corresponding minimal problems are formulated in terms of polynomial equations, and the hidden variable method is used to efficiently find their solutions. We also choose among several existing formulations of the relative pose problem to make the estimators as efficient as possible. The performance of the proposed algorithms is evaluated in a series of experiments on synthetic and real data.

%\keywords{Multiview geometry \and Relative pose estimation \and $\SE(3)$ invariants \and Relative rotation angle \and Screw translation}

\end{abstract}

\section{Introduction}
\label{sec:intro}
Minimal relative pose solver of a camera is a fundamental component in modern 3D vision applications including robot localization and mapping, augmented reality, autonomous driving, 3D modeling, etc. Well-known solvers include the 7-point algorithm~\cite{HZ} and the 5-point algorithm~\cite{Nister,SEN06}. It is generally admitted that an $n$-point solver with smaller $n$ performs more robustly, has less degenerate configurations and requires less iterations when integrated in a RANSAC framework.

% This paper considers two scenarios where the minimal number of points can be further reduced: 
% \begin{enumerate}
%     \item A system with a camera and an orientation sensor such as an IMU.
%     \item A system with a camera and moves on a plane.
% \end{enumerate}
% In practice, these scenarios widely exist in recent applications of indoor robots, autonomous vehicles and UAVs. 

As the first contribution of this paper, we show that two measurements -- rotation angle and screw translation -- can be respectively integrated into relative pose solvers to reduce the number $n$ of minimal points. Typical scenarios for these measurements include a robot equipped with a camera and an IMU, and a robot with planar motion. These measurements are referred as $\SE(3)$ invariants as they stay invariant cross different coordinate systems on a rigid body. Consequently, the proposed methods do not require known extrinsic pose of the camera with respect to the IMU or the motion plane, which is an important advantage over previous relative pose estimation methods \cite{FTP10,Kalantari-vert,LHLP,Saurer2017,Choi18,Scaramuzza11,Lee-vert}. This advantage make the proposed methods more flexible and convenient. For example, when estimating visual odometry to hand-eye calibrate the camera-IMU extrinsics, the proposed methods improve trajectory estimation even though the extrinsics are unavailable. For robot systems subjected to long term operation, the proposed methods avoid re-calibration. All the 5 minimal problems introduced by different combination of $\SE(3)$ invariants as constraints are comprehensively studied in this paper.

The second contribution is a comprehensive study of all existing polynomial formulations for relative pose solvers. We show pros and cons of each formulation under different minimal problem settings and reveal connections between the formulations. For each proposed relative pose problem with $\SE(3)$ invariants, we evaluate these formulations and propose solvers with the best efficiency.

\section{Related Works}

% We first briefly overview previous works on relative pose estimation.
A fundamental matrix for a pair of pinhole cameras has 7 DoF and can be estimated minimally from 7 and linearly from 8 point correspondences~\cite{hartley95,HZ,longuet81}. If all camera intrinsics except a common focal length are calibrated, then the estimation can be reduced to the 6-point algorithm~\cite{Hartley2012,KBP08,SNKS}. If the focal length is also calibrated, the 5-point algorithm~\cite{Nister,SEN06} is naturally introduced.

Beyond 5-point solution, extra constraints can be exploited. With known gravity direction measured by IMU or by the knowledge of motion plane, \cite{FTP10,Kalantari-vert,Saurer2017,Choi18} obtain two rotation angles and hence reduce the minimal number of point pairs to~$3$. Camera extrinsics are required to be calibrated for these methods to transform the two angles to the camera frame. In~\cite{Scaramuzza11}, it is assumed that the camera follows the Ackermann motion. In this case only one point is needed for pose estimation. This method requires the camera to be specifically mounted. Paper~\cite{LHLP} proposes the 4-point algorithm given a known rotation angle measurement from other sensors. This is the first work on integrating $\SE(3)$ invariants in relative pose estimation. The known rotation angle can also be used for the camera self-calibration as demonstrated in~\cite{Mart18}. As mentioned in the previous section, extrinsic calibration is not required to use $\SE(3)$ invariants. 

Minimal problems are usually formulated in terms of multivariate polynomial systems and a plenty of methods have been proposed to solve these systems. Some of these methods make use of the Gr\"obner basis computation~\cite{KBP08,Kneip12,Mart18,MartLi19,SEN06}. The roots are then derived from the eigenvectors of the so-called action matrix constructed from the Gr\"obner basis~\cite{CLS}. Besides action matrix, alternative matrix decomposition methods were also proposed including PolyEig~\cite{KBP} and QuEst~\cite{Ramirez18}. To avoid significant computational cost of matrix decomposition, the hidden variable approach has been used in several solvers~\cite{Hartley2012,Nister}. This approach reduces the problem to finding real roots of a univariate polynomial.

\section{Preliminaries}

\subsection{Notation}

We preferably use $\alpha$, $\beta$, \dots for scalars, $a$, $b$, \dots for column 3-vectors, and $A$, $B$, \dots for matrices. For a matrix $A$, the transpose is $A^\top$, the determinant is $\det A$, and the trace is $\tr A$. For two 3-vectors $a$ and $b$ the cross product is $a \times b$. For a vector $a$, the entries are $a_i$, the notation $[a]_\times$ stands for the skew-symmetric matrix such that $[a]_\times b = a \times b$ for any vector~$b$. We use $I$ for the identity matrix and $\| \cdot \|$ for the Frobenius norm. A notation $\mathsf{f}^*_*$ is used to refer a polynomial.

% \subsection{Rotation Matrices}

A rotation matrix $R$ can be represented by a unit quaternion $[\sigma\ u^\top]$ as follows
\begin{equation}
R = 2 (u u^\top - \sigma [u]_\times) + (\sigma^2 - \|u\|^2) I,
\label{eq:quat-to_rot}
\end{equation}
where
\begin{equation}
\mathsf{f}^\sigma := \|u\|^2 + \sigma^2 - 1 = 0.
\label{eq:unit-quaternion}
\end{equation}
With $\theta$ as the rotation angle, we have
\begin{equation}
\tr R = 4 \sigma^2 - 1 = 2 \cos\theta + 1,
\label{eq:rotation-trace}
\end{equation}

Another way to represent a rotation matrix $R$ comes from the Cayley transform if and only if it is not a rotation through an angle $\pi + 2\pi k$ for a certain integer~$k$.
\begin{equation}
R = (I - [v]_\times)(I + [v]_\times)^{-1},
\label{eq:cayley}
\end{equation}
where $v = u/\sigma$ is a 3-vector.

The special Euclidean group $\SE(3)$ consists of all orientation-preserving rigid motions of 3-dimensional Euclidean space. Any element $H \in \SE(3)$ can be represented by a $4\times 4$ matrix of the form
\begin{equation}
H = \begin{bmatrix}R & t \\ 0^\top & 1\end{bmatrix},
\label{eq:matrixH}
\end{equation}
where $R \in \mathrm{SO}(3)$ and $t \in \mathbb R^3$ are the rotational and translational parts of $H$ respectively. In the sequel, saying about elements of group $\SE(3)$ we always imply $4\times 4$ matrices of type~\eqref{eq:matrixH}.

\subsection{Epipolar Constraint}

Let $P' = [R' \ t']$ and $P'' = [R'' \ t'']$, where $R', R'' \in \mathrm{SO}(3)$ and $t', t'' \in \mathbb R^3$, be calibrated camera matrices. Let $q_i'$ and $q_i''$ be the corresponding images of a 3D point~$Q_i$. Then the epipolar constraint reads
\begin{equation}
\mathsf{f}_i := q_i''^\top (R [t']_\times - [t'']_\times R) q_i' = 0,
\label{eq:pepc}
\end{equation}
where $i$ counts the point pairs and $R = R''R'^\top$ is called the relative rotation matrix. We notice that Eq.~\eqref{eq:pepc} can be rewritten in form
\begin{equation}
q_i''^\top [t]_\times R q_i' = 0,
\label{eq:pepc2}
\end{equation}
where $t = Rt' - t''$ is called the relative translation. Matrix $E = [t]_\times R$ is well known in the computer vision community as an essential matrix.

\begin{SCfigure}[3]
    \centering
    \includegraphics[width=0.2\textwidth]{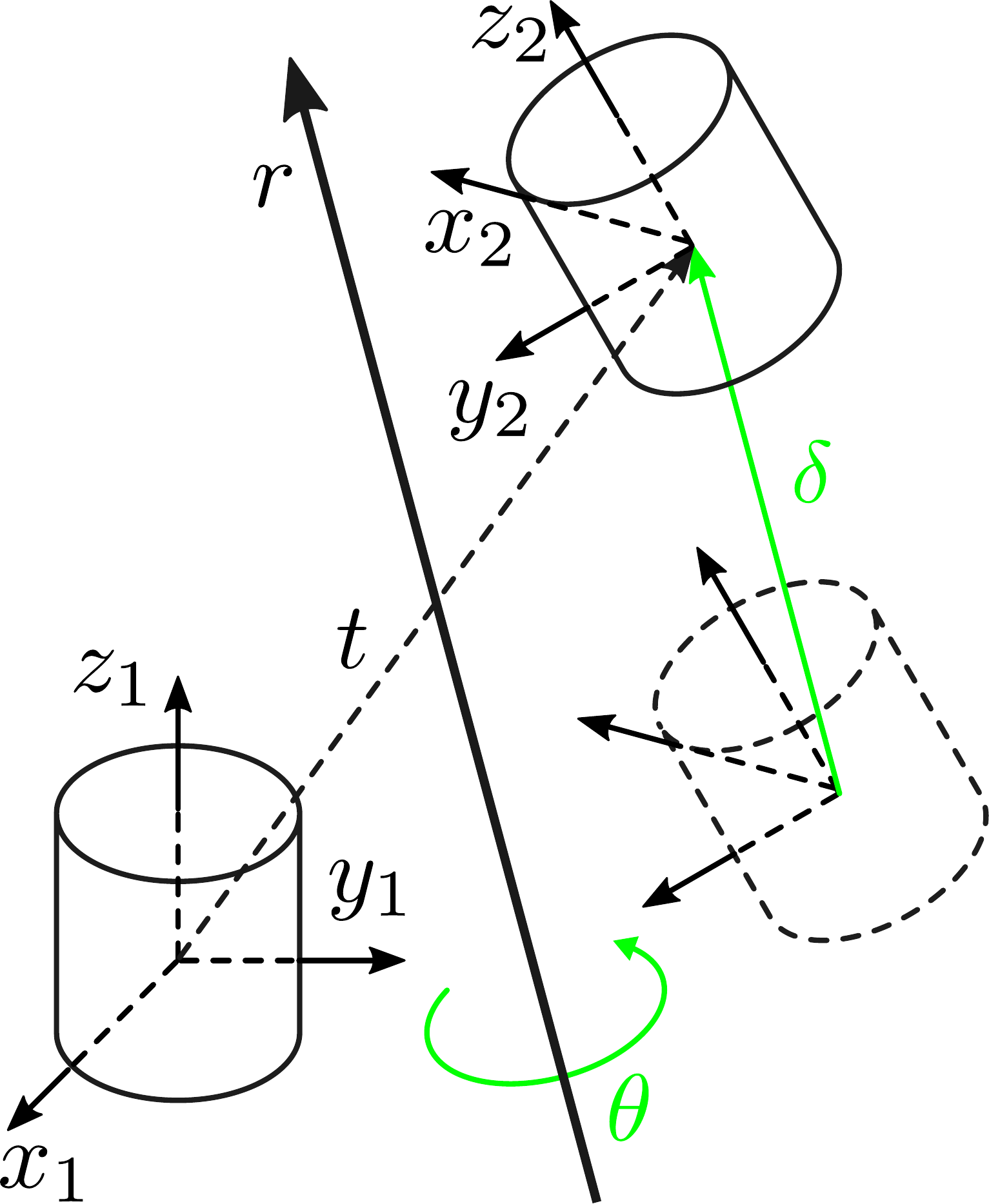}
    \caption{A rigid motion in $\SE(3)$ can always be decomposed as a rotation around an axis $r$ with angle $\theta$, and a screw translation $\delta$ along $r$. $\theta$ and $\delta$ remain consistent for different parts of the rigid body, regardless of the part offset and the local coordinate system.}
    \label{fig:se3}
\end{SCfigure}

\subsection{$\SE(3)$ invariants}

Given an element $H \in \SE(3)$, with its rotational part $R$ represented by~\eqref{eq:quat-to_rot}. Denote by $r = \frac{u}{\|u\|}$ the unit rotation axis of~$R$. Then the value
\begin{equation}
\delta = r^\top t
\label{eq:screw-def}
\end{equation}
is the \textit{screw translation} of $H$.
In this paper, we are specifically interested in the case of $\delta = 0$, which is also equivalent to
\begin{equation}
\mathsf{f}^0 := u^\top t = 0.
\label{eq:f0}
\end{equation}
Consider a robot with planar motion. Its rotation axis $r$ must be the normal vector of the motion plane. Its translation vector $t$ must lie on the motion plane. Thus it is obvious that the condition of zero screw translation ($\delta = 0$) holds for any planar motion regardless of the camera orientation with respect to the ground plane direction.

Figure~\ref{fig:se3} illustrates the definition of $\theta$ and $\delta$. We refer them as the $\SE(3)$ invariants, i.e. scalar values invariant under the conjugation by an $\SE(3)$ element. In robotics, this conjugation is known as the hand eye transformation. The difference between $\SE(3)$ invariants and an easily mixed-up concept bi-invariant metrics can be found from~\cite{Chirikjian15}.

\begin{theorem}[$\SE(3)$ Invariants]
For a transform $H \in \SE(3)$, its rotation angle $\theta$ and screw translation $\delta$ are invariant under the hand eye transformation $H' = X^{-1} H X$ with $X \in \SE(3)$.
\label{thm:invariants}
\end{theorem}

\begin{proof}
Let the rotational and translational parts of $H$, $H'$, $X$ be $R$, $R'$, $R_X$ and $t$, $t'$, $t_X$ respectively. Then we have
\begin{multline}
H' = \begin{bmatrix}R_X^\top & -R_X^\top t_X \\ 0^\top & 1\end{bmatrix} \begin{bmatrix}R & t \\ 0^\top & 1\end{bmatrix} \begin{bmatrix}R_X & t_X \\ 0^\top & 1\end{bmatrix} = \begin{bmatrix}R_X^\top R R_X & R_X^\top (R t_X - t_X + t) \\ 0^\top & 1\end{bmatrix},
\end{multline}
that is $R' = R_X^\top R R_X$ and $t' = R_X^\top (R t_X - t_X + t)$. The invariance of $\theta$ follows from Eq.~\eqref{eq:rotation-trace}, since $\tr(R') = \tr(R_X^\top R R_X) = \tr(R)$.

Further, let $r$ and $r'$ be the unit rotation axes of $R$ and $R'$ respectively. It is clear that $r = R r$ and $r' = R' r' = R_X^\top R R_X r'$ (Lemma II \cite{tsai1988real}). Hence, axes $r$ and $r'$ are related by $r' = R_X^\top r$. Substituting this into the definition of $\delta'$ yields
\begin{multline}
\delta' = r'^\top t' = r^\top R_X R_X^\top (R t_X - t_X + t) = r^\top R t_X - r^\top t_X + r^\top t = r^\top t = \delta.
\end{multline}
% This completes the proof.
\end{proof}
Theorem~\ref{thm:invariants} is a well-known result in robotics~\cite{Chen91,Chirikjian15}. Different proof for the rotation part can also be found from~\cite{shiu1989calibration,Chen91,Chirikjian15}.
% Consider the scenarios mention in Section~\ref{sec:intro}:
% \begin{enumerate}
%     \item For a system with a camera and an angle sensor, Theorem~\ref{thm:invariants} implies that rotation angle measured by the angle sensor is always the rotation angle of the camera.
%     \item For a system of planar motion, Theorem~\ref{thm:invariants} implies that its equipped camera also moves on a plane and have zero screw translation.
% \end{enumerate}

\section{Minimal Problem Formulations}

The relative pose estimation problem aims to solve for the relative rotation $R$ and relative translation $t$ given several image point pairs. It is well known that if no additional constraints are used, the relative pose can be estimated minimally from $5$ point pairs \cite{Nister,SEN06,Hartley2012}. With known $\SE(3)$ invariants (rotation angle $\theta$ and screw translation $\delta$), the number of point pairs required for a minimal solution is reduced.
Table \ref{tab:problem-summary} summarizes the minimal relative pose estimation problems that can be formulated for different combinations of image data and $\SE(3)$ invariants. 
%In this paper, we consider the following minimal relative pose estimation problems:
%\begin{itemize}
%\item 5P: given 5 point pairs;
%\item 4P-RA: given 4 point pairs and angle $\theta$;
%\item 4P-ST0: given 4 point pairs and constraint $\delta = 0$;
%\item 3P-RA-ST0: given 3 point pairs, angle $\theta$ and constraint $\delta = 0$;
%\item 5PST1: given 5 point pairs and constraint $\delta \neq 0$;
%\item 4P-RA-ST1: given 4 point pairs, angle $\theta$ and constraint $\delta \neq 0$.
%\end{itemize}

\begin{SCtable}[]
\centering\footnotesize
\begin{tabular}{llllll}
\hline
Problem & $\SE(3)$ Inv& $n$ & DoF & Variables & Constraints \\ \hline
5P \cite{Hartley2012,Nister,SEN06} & - & 5 & 5 & $\sigma, u, t$ & $\mathsf{f}_1, \mathsf{f}_2, \mathsf{f}_3, \mathsf{f}_4, \mathsf{f}_5$\\
4P-RA \cite{LHLP,MartLi19} & $\theta$ & 4 & 5 & $u, t$ & $\mathsf{f}_1, \mathsf{f}_2, \mathsf{f}_3, \mathsf{f}_4, \mathsf{f}^\sigma$ \\
4P-ST0 & $\delta = 0$ & 4 & 5 & $\sigma, u, t$ & $\mathsf{f}_1, \mathsf{f}_2, \mathsf{f}_3, \mathsf{f}_4, \mathsf{f}^0$ \\
3P-RA-ST0 & $\theta, \delta = 0$ & 3 & 5 & $u, t$ & $\mathsf{f}_1, \mathsf{f}_2, \mathsf{f}_3, \mathsf{f}^\sigma, \mathsf{f}^0$ \\
5P-ST1 & $\delta \neq 0$ & 5 & 6 & $\sigma, u, t$ & $\mathsf{f}_1, \mathsf{f}_2, \mathsf{f}_3, \mathsf{f}_4, \mathsf{f}_5$ \\
4P-RA-ST1 & $\theta, \delta \neq 0$ & 4 & 6 & $u, t$ & $\mathsf{f}_1, \mathsf{f}_2, \mathsf{f}_3, \mathsf{f}_4, \mathsf{f}^\sigma$\\ \hline
\end{tabular}
    \caption[]{Relative pose problems with $\SE(3)$ invariants.\\
{\scriptsize RA:~Relative angle ($\theta$)\\
    ST0:~Zero screw translation ($\delta = 0$)\\
    ST1:~Non-zero screw translation ($\delta \neq 0$)\\
$n$:~Number of points for minimal cases}
    }
\label{tab:problem-summary}
\end{SCtable}

\begin{remark}
    In 5P, 4P-RA, 4P-ST0 and 3P-RA-ST0, each of $\SE(3)$ invariants $\mathsf{f}^\sigma$ and $\mathsf{f}^0$ can replace one of point correspondences $\mathsf{f}_i$ as constraints and the relative pose can be estimated up to an ambiguous scale (5 DoF). In 5P-ST1 and 4P-RA-ST1, condition~\eqref{eq:screw-def} with $\delta \neq 0$ can not be used to replace point correspondences as the essential matrix does not change regardless of the value of a non-zero $\delta$. Instead $\delta$ can be used to determine the length of translation and hence the overall scale is observable (6 DoF). 5P-ST1 and 4P-RA-ST1 are therefore equivalent to problems 5P and 4P-RA respectively.
\end{remark}

%\subsection*{Beyond Pinhole Cameras}

\begin{remark}[Beyond Pinhole Cameras]
In this paper we only focus on relative pose estimation for a pinhole camera. However, it is worth mentioning the difference between pinhole camera and generalized camera models under known $\SE(3)$ invariants. For a generalized camera model~\cite{Pless,SNOA}, the relative translation length is observable. The vanilla version of relative pose estimation problem requires $6$ points to fully recover the 6 DoF relative pose. When the screw translation is known, regardless of being zero or non-zero, 5 points are required to fully recover the 6 DoF relative pose. The case of known relative rotation angle for generalized cameras was covered in~\cite{MartLi19}.
\end{remark}

\section{Solution Formulations for Relative Pose Estimation}
\label{sec:formulations}

Various solutions have been studied in the past decades on relative estimation problems, i.e. to solve~\eqref{eq:pepc} or~\eqref{eq:pepc2} under different constraints. We provide a comprehensive summary for these formulations in this section and also discover how $\SE(3)$ invariants can be denoted for each formulation. All mentioned previous formulations are also listed in Table~\ref{tab:previous}.
% We briefly recall some existing formulations of relative pose problems in terms of polynomial equations.
%In our study, we found that there is no a universal formulation that would be optimal for all relative pose problems.

\begin{table}
\begin{minipage}[t]{0.55\textwidth}
    \centering\footnotesize
    \begin{threeparttable}
	\begin{tabular}[t]{lllrrr}
	    \hline
            Problem & Form & \#$R$ & \#$t$ & Templ & \#S\\\hline
            PC 5P~\cite{Kneip12} & SIR3 & 9 ($R$) & 3 & $66\PLH197$ & 20\\
            PC 5P~\cite{Fathian18} & SIR6 & 4 ($u, \sigma$) & 3 & $40\PLH56$ & 35\\
            PC 5P~\cite{kalantari2009five} & Direct & 4 ($u, \sigma$) & 3 & NR & 80\\
            PC$+\theta$ (4P-RA)~\cite{LHLP} & SIR3 & 3 ($u$) & 3 & $270\PLH290$ & 20\\
            PC$+\theta$ (4P-RA)~\cite{MartLi19} & SIR2 & 3 ($u$) & 2 & $16\PLH36$ & 20\\
            PC$+$vert (IMU)~\cite{Kalantari-vert} & SIR3 & 1 (yaw) & 3 & CF & 4\\
            PC$+$vert (planar)~\cite{Choi18} & Direct & 1 (yaw) & 1 & CF & 4\\
            PC$+$Ackermann~\cite{SF} & Direct & 1 (yaw) & 0 & CF & 1\\
            GC 6P~\cite{SEN06} & SIR2 & 3 ($v$) & 2 & $60\PLH120$\tnote{1} & 64\\
            GC$+\theta$~\cite{MartLi19} & SIR2 & 3 ($u$) & 2 & $37\PLH81$ & 44\\
            GC$+$vert (IMU)~\cite{Lee-vert} & SIR3 & 1 (yaw) & 3 & CF & 8\\
            GC$+$vert (planar)~\cite{lee2013structureless} & SIR3 & 1 (yaw) & 2 & CF & 6\\
            GC$+$Ackermann~\cite{hee2013motion} & Direct & 1 (yaw) & 1 & CF & 2\\
            \hline
        \end{tabular}
    \end{threeparttable}
\end{minipage}
\hspace{11pt}
\begin{minipage}[t]{0.38\textwidth}
    \centering\footnotesize
    \begin{threeparttable}
	\begin{tabular}[t]{llrr}
	    \hline
            Problem & Form & Templ & \#S\\\hline
            PC 5P~\cite{Nister,SEN06} & NullE & $10\PLH20$ & 10\\
            PC$-$focal~\cite{KBP08,SNKS} & NullE & $31\PLH46$ & 15\\
            PC$+\theta-$focal$-$pp~\cite{Mart18} & NullE & $19\PLH32$\tnote{2} & 6\\
            PC$+$vert (IMU)~\cite{FTP10} & NullE & $6\PLH10$ & 4\\
            \hline
        \end{tabular}
        \begin{tablenotes}
        \vspace{3pt}
            \scriptsize
        \item[*] \#$R$, \#$t$: Number of parameters
        \item[*] \#S: Number of solutions
        \item[*] $+$/$-$: With a constraint or an unknown
        \item[*] PC/GC: Pinhole camera / generalized camera
        \item[*] NR: Not reported
        \item[*] CF: Closed form solution w/o template matrix
        \item[*] vert: Vertical direction
        \item[*] Ackermann: Ackermann motion model
        \item[*] focal/pp: Focal length / principle point
        \item[1] Insufficient for full Gr\"obner basis generation
        \item[2] The largest of cascaded templates reported
    \end{tablenotes}
    \end{threeparttable}
\end{minipage}
	\vspace{6pt}
        \caption[]{Representitive works of polynomial formulations for relative pose estimation.}
        \label{tab:previous}
\end{table}

\subsection{Solutions by Decomposing $E$}
\subsubsection{Directly Solving $R$ and $t$.}

As an intuitive start, it is possible to directly solve~\eqref{eq:pepc},~\eqref{eq:pepc2} by considering $R$ and $t$ as polynomial unknowns. In \cite{LHLP}, $R$ is parameterized by angle-axis to integrate rotation angle and $t$ is constrained by $\|t\|^2 = 1$ to remove scale ambiguity. In theory $R$ can be also parameterized by a quaternion with constraint of $\mathsf{f}^\sigma$ or by $9$ matrix elements with constraint of $R^\top R = I$. However these formulations involve $6\sim12$ unknowns (including 3 for translation) and require quite complicated polynomial elimination, which makes high computational burden in real-time applications. A simpler specialization is when the vertical direction is known from IMU \cite{Kalantari-vert} or known ground plane \cite{Choi18}, $R$ can be parameterized by a yaw angle rotation. With Ackermann motion assumed, the parameters can even be further reduced~\cite{hee2013motion,SF}.

Since each $\mathsf{f}_i$ is linear in $t$, cf. Eq.~\eqref{eq:pepc2}, several formulations have been proposed in previous works to eliminate translation variables and solve rotation parameters only for simplicity.

\subsubsection{SIR3: Solving Isolated Rotation by a $3\PLH3$ Determinant.}

We can isolate unknown translation by rewriting the epipolar constraints~\eqref{eq:pepc2} for $n$ point pairs in the form
\begin{equation}
G \ t = 0,
\label{eq:Gt}
\end{equation}
where $G$ is a matrix of size $n \times 3$. Elements of the $i$-th row of matrix $G$ are polynomials in unknown rotation parameters and known $q_i', q_i''$. It follows from Eq.~\eqref{eq:Gt} that all $3\times 3$ minors of $G$ must vanish for a valid translation. Thus we obtain new polynomial constraints of the total degree~$6$ on the rotation parameters only. This solution formulation will be further referred to as \textit{SIR3}. If the rotation matrix is represented by~\eqref{eq:quat-to_rot} (resp. by~\eqref{eq:cayley}), then the formulation is denoted by SIR3+$u$ (resp. SIR3+$v$). When rotation is constrained to have only one unknown, SIR3 generates a closed form univariate polynomial \cite{Kalantari-vert,Lee-vert}. However, with more unknowns in rotation, it is still not satisfactory as leads to large matrix templates for the Gr\"obner basis computation. For example, in~\cite{Kneip12} SIR3 is used to solve 5P by a reduction on a $66\times 197$ matrix. In~\cite{LHLP}, the SIR3+$u$ formulation of 4P-RA involves even larger template matrix and has to be solved by numerical search.

The known rotation angle can be easily integrated into SIR3+$u$. If the motion is planar, i.e. the screw translation is zero, polynomial $\mathsf{f}^0$ from~\eqref{eq:f0} can be written as a new row $u^\top$ of matrix~$G$. 

%In Table~\ref{tab:template-size}, we show that AG produces a $290 \times 330$ elimination template for SIR3 form.

%In the derivation of~\cite{SNOA}, the unknown rotation $[\sigma\ u^\top]$ in SIR3 is replaced by a scaled version $[1\ v^\top]$ to reduce one more unknown while not changing the zero value of $\det F_i^{jk}$. For example, as shown in Table~\ref{tab:template-size}, the elimination template size is further reduced for 5P and 4P-ST0 whose $\sigma$ is unknown. For 4P-RA and 3P-RA-ST0 the two versions are obviously equivalent.

\subsubsection{SIR6: Solving Isolated Rotation by a $6\PLH6$ Determinant.}
Let $Q_i$ be the $i$-th 3D point so that 
\begin{equation}
\lambda_i q_i' = P' Q_i,\quad \mu_i q_i'' = P'' Q_i,
    \label{eq:a-point-pair}
\end{equation}
where $\lambda_i$ and $\mu_i$ are some scalars. The relative pose $R$ and $t$ satisfies $\lambda_i q_i'' = \mu_i R q_i' + t$. Consider this equation for the $i, j, k$-th points and subtract the $i$-th equation over the $j$-th and $k$-th respectively to eliminate $t$. We can obtain two 3D linear equations, forming a $6\times6$ matrix $M_{ijk}$.
\begin{equation}
    M_{ijk} [\lambda_i\ \mu_i\ \lambda_j\ \mu_j\ \lambda_k\ \mu_k]^\top = 0.
\end{equation}
Similar to SIR3, the determinant of $M_{ijk}$ must vanish, resulting in polynomials on the rotation parameters only. SIR6 is proposed by \cite{Fathian18} as an alternative solution to 5P. Symbolic computation reveals the relationship between SIR6 and SIR3:
\begin{theorem}
    Consider a $3 \times 3$ submatrix $G_{ijk}$ of matrix $G$ whose three rows correspond to the $i$-th, $j$-th, and $k$-th point pairs. We have $\det G_{ijk} = \det M_{ijk}$, up to a sign.
    \label{thm:t6d}
\end{theorem}

\subsubsection{SIR2: Solving Isolated Rotation by a $2\PLH2$ Determinant.}
\label{sec:t2d}

Using the rigid motion ambiguity of the world coordinate frame, we set $Q_i =[0 \ 0 \ 0 \ 1]^\top$ in~\eqref{eq:a-point-pair} for a certain~$i$. This yields
\begin{equation}
t' = \lambda_i q_i',\ t'' = \mu_i q_i''.
\end{equation}
Substituting $t'$ and $t''$ into Eq.~\eqref{eq:pepc} for a $j$-th pair with $j \neq i$, we convert Eq.~\eqref{eq:pepc} into $F_{ij} \begin{bmatrix} \lambda_i & \mu_i \end{bmatrix}^\top = 0$. Construct $F_{ik}$ from a $k$-th point pair and stack with $F_{ij}$ as $F_{ijk}$, we have
% \begin{equation}
% \begin{bmatrix} q_j''^\top R p_{ij}' & p_{ij}''^\top R q_j' \end{bmatrix}
%     \begin{bmatrix} \lambda_i & \mu_i \end{bmatrix}^\top = 0,
% \label{eq:simpler-equation}
% \end{equation}
\begin{equation}
    F_{ijk} \begin{bmatrix} \lambda_i & \mu_i \end{bmatrix}^\top = 0.
\label{eq:generalized-matrix}
\end{equation}
% where $p_{ij}' = q_i' \times q_j'$ and similarly for $p_{ij}''$. Thus, given $n$ point pairs, we obtain $n - 1$ equations of type~\eqref{eq:simpler-equation}. Let these equations be stacked as
where matrix $F_{ijk}$ is of size $2\times2$. $F_{ijk}$ must have zero determinant. This leads to degree $4$ polynomial equations in the rotation parameters. The proposed solution formulation will be further referred to as \textit{SIR2}. The two versions of SIR2 corresponding to the quaternion~\eqref{eq:quat-to_rot} and Cayley~\eqref{eq:cayley} parametrizations are denoted by SIR2+$u$ and SIR2+$v$ respectively. The SIR2+$v$ form was earlier used in~\cite{SNOA} for solving the relative pose problem for generalized cameras. In~\cite{MartLi19}, SIR2+$u$ was used to solve 4P-RA.
%Taking 5P as an example, AG produces a $90\times130$ elimination template for SIR2 as shown in Table~\ref{tab:template-size}.
Symbolic computation also reveals the connection between SIR2 and SIR3 and explains why SIR2 is a simpler formulation compared to SIR3/SIR6.
\begin{theorem}
    Consider $G_{ijk}$ as defined in Theorem~\ref{thm:t6d}. Under SIR2/SIR3+$u$, we have:
    \begin{equation}
    \det G_{ijk} = (\|u\|^2 + \sigma^2) \cdot \det F_{ijk},
    \label{eq:factorization}
    \end{equation}
    up to a sign. This equation also holds if replacing $(\|u\|^2 + \sigma^2)$ with $(\|v\|^2 + 1)$ under the SIR2/SIR3+$v$.
\end{theorem}

\begin{remark}
According to Eq.~\eqref{eq:factorization}, since equation $\|v\|^2 + 1 = 0$ has infinite number of complex solutions, 5P in SIR3+$v$ is not zero-dimensional over~$\mathbb C$.
\end{remark}

The known rotation angle can be easily integrated into SIR2+$u$. If the motion is planar, i.e. $\delta = 0$, then polynomial $\mathsf{f}^0$ can add a row $[- u^\top q_i' \ u^\top q_i'']$ to $F_{ijk}$.

% Now we would like to establish a connection between the SIR2 and SIR3 formulations. Consider a $2 \times 2$ submatrix $F_i^{jk}$ of matrix $F_i$ whose two rows correspond to the $j$-th and $k$-th point pairs. It was proved in~\cite{MartLi19} that $\det F_i^{jk}$ does not change (up to a sign) when permuting indices $i$, $j$, and~$k$. Consider also a $3 \times 3$ submatrix $G_{ijk}$ of matrix $G$ whose three rows correspond to the $i$-th, $j$-th, and $k$-th point pairs. Then, for the quaternion representation~\eqref{eq:quat-to_rot} of matrix $R$, symbolic computation yields
% Similar conclusion is also discovered by~\cite{Fathian18}, who obtains $\det G_{ijk}$ in a different formulation. This equation also holds if replacing $(\|u\|^2 + \sigma^2)$ with $(\|v\|^2 + 1)$ under the SIR3+$v$ Caylay representation~\eqref{eq:cayley}.

%We would like to mention that in~\cite{Fathian18} another formulation of the 6-degree polynomial determinant is proposed, and this determinant is also factorizable by $\|u\|^2 + \sigma^2$. The formulation can be verified to be equivalent to $\det G_{ijk}$.

\subsection{Solutions by Constraining $E$}
\subsubsection{NullE: Solving Essential Matrix Represented by NullSpace Bases.}

Instead of direct solving for $R$ and $t$, a more classical approach to the relative pose problem is solving first for the essential matrix $E$ which is a mixed form of rotation and translation parameters. Unknown $E$ is parameterized by $\sum_{i = 1}^{9 - n} \gamma_i E^{(i)}$, where matrices $E^{(i)}$ form the nullspace basis of the underdetermined linear system $\{\mathsf{f}_i \mid i = 1, \dots, n \}$, $\gamma_i$ are new unknowns which are usually scaled so that $\gamma_1 = 1$. Traditional 5P solvers \cite{Nister,SEN06,Hartley2012} use the following constraints to form a polynomial system on $\gamma_i$:
\begin{align}
\label{eq:det-e}
\det E &= 0,\\
\label{eq:eee}
2EE^\top E - \tr(EE^\top) E &= 0.
\end{align}
In addition, \cite{FTP10} found that known vertical can be denoted as constraints on $E$.
Known $\SE(3)$ invariants also can be formulated as constraints on~$E$ as follows.

\begin{theorem}[\cite{Mart18}]
Let $E = [t]_\times R$ be an essential matrix and $\tr R = \tau$. Then $E$ fulfills the following equation
\begin{equation}
\frac{1}{2}\,(\tau^2 - 1)\tr(EE^\top) + (\tau + 1)\tr(E^2) - \tau\tr^2 E = 0.
\label{eq:e-tau}
\end{equation}
\end{theorem}
% libo: Very weird, in llns template, two consecutive theorems will eat the header of the second one. Add this rule to workaround.
\rule{0pt}{0pt}
\begin{theorem}
\label{thm:tr-E}
Let $E = [t]_\times R$ be an essential matrix and $R$ be a rotation through an angle $\theta$ around a vector $r \neq 0$. If $\delta = r^\top t = 0$, then
\begin{equation}
\tr E = 0.
\label{eq:e-trace}
\end{equation}
Conversely, if $\tr E = 0$, then either $\delta = 0$ or $\theta = \pi k$ for a certain integer~$k$.
\end{theorem}

\begin{proof}
% We parameterize matrix~$R$ by Rodrigues' formula
% \[
% R = \cos\theta I + (1 - \cos\theta)rr^\top + \sin\theta [r]_\times,
% \]
% which is an equivalent form of Eq.~\eqref{eq:quat-to_rot}. Then,
% \begin{multline}
% \tr E = \tr([t]_\times R) \\= (1 - \cos\theta) \tr([t]_\times rr^\top) + \sin\theta \tr([t]_\times [r]_\times).
% \end{multline}
% Since
% \begin{align}
% \tr([t]_\times rr^\top) = -\tr(r^\top[r]_\times t) &= 0,\\
% \tr([t]_\times [r]_\times) = \tr(rt^\top) - r^\top t \tr(I) &= -2(r^\top t),
% \end{align}
We utilize Proposition 2.20 from \cite{maybank2012theory} that $\tr E = -2\sin\theta(r^\top t)$ and the statement follows.
\end{proof}

Using Eqs.~\eqref{eq:det-e}~--~\eqref{eq:e-trace}, problems 4P-RA, 4P-ST0, and 3P-RA-ST0 can be formulated in terms of matrix~$E$.
Table~\ref{tab:template-size} summarizes these formulations with reference template matrix size generated by automatic polynomial solver generators \cite{KBP08} and \cite{li2020gaps}.

\begin{SCtable}[1][t]
\centering\footnotesize
\begin{tabular}{llrrrrrr}
\hline
Problem & Form & \#V & D & \#S & AG & GAPS & Proposed\\
\hline
5P & SIR3+$u$ & $4$ & $6$ & $40$ & $146 \PLH 186$ & $116 \PLH 200^1$ & - \\
5P & SIR3+$v$ & $3$ & $6$ & $\infty$ & - & - & - \\
5P & SIR2+$u$ & $4$ & $4$ & $40$ & $90 \PLH 130$ & $60 \PLH 80^1$ & - \\
5P & SIR2+$v$ & $3$ & $4$ & $20$ & $31 \PLH 51$ & $36 \PLH 56$ & - \\
5P & NullE & $3$ & $3$ & $10$ & $10 \PLH 20$ & $10 \PLH 20$ & - \\\hline
4P-RA & SIR2+$u$ & $3$ & $4$ & $20$ & $26 \PLH 46$ & $36 \PLH 56$ & $16\PLH36$~\cite{MartLi19} \\
4P-RA & NullE & $4$ & $3$ & $20$ & $34\PLH 54^2$ & $50 \PLH 70$ & - \\\hline
4P-ST0 & SIR2+$u$ & $3$ & $4$ & $20$ & $62 \PLH 82^2$ & $38 \PLH 65^1$ & - \\
4P-ST0 & SIR2+$v$ & $3$ & $4$ & $10$ & $25 \PLH 35$ & $27 \PLH 35$ & - \\
4P-ST0 & NullE & $3$ & $3$ & $10$ & $10 \PLH 20$ & $10 \PLH 20$ & $10 \PLH 20$ \\\hline
3P-RA-ST0 & SIR2+$u$ & $3$ & $4$ & $12$ & $23 \PLH 35^2$ & $28 \PLH 35$ & $13 \PLH 25$ \\
3P-RA-ST0 & NullE & $5$ & $3$ & $20$ & $34 \PLH 54$ & $50 \PLH 70$ & - \\
3P-RA-ST0 & NullEx & $5$ & $3$ & $12$ & $22 \PLH 35^2$ & $53 \PLH 65$ & - \\
\hline
\end{tabular}
    \caption[]{Comparison of different solution formulations for the minimal relative pose problems with $\SE(3)$ invariants. \\
{\scriptsize \#V:~Number of variables\\
    D:~ Highest degree\\
    \#S:~Number of solutions\\
    AG:~Generator from~\cite{KBP08}\\
    GAPS:~Generator from~\cite{li2020gaps}\\
    $1$:~Mirrored roots $u$ merged by~\cite{larsson2016uncovering}\\
    $2$:~The largest of cascaded templates reported}
}
\label{tab:template-size}
\end{SCtable}

%\item In case $\sigma$ is known the SIR2 form generates simpler solutions, whereas the NullE form has the lowest polynomial degree, but introduces more variables, which increases the problem complexity.

\begin{remark}
As is well known, the 5P problem in the NullE formulation has $10$ solutions. Each essential matrix corresponds to a twisted pair of rotations~\cite{HZ} and each rotation, being represented by a unit quaternion, doubles due to the sign ambiguity. Therefore, as shown in Table~\ref{tab:template-size}, the 5P problem in SIR2/SIR3+$u$ has $40$ solutions while in SIR2/SIR3+$v$ has only $20$ solutions.
The 4P-RA problem in the NullE formulation has 20 solutions, corresponding to $40$ rotations. For each pair of rotations, there is a unique one whose rotation angle equals known~$\theta$. Similarly for 4P-ST0 in NullE, each pair of rotations corresponding to an essential matrix contains a unique valid rotation.
\end{remark}

\begin{remark}
\label{rem:7cubics}
The 3P-RA-ST0 problem has $12$ solutions in SIR2+$u$. However, in NullE the system consisting of Eqs.~\eqref{eq:det-e}~--~\eqref{eq:e-trace} has $20$ solutions. The obtained contradiction indicates that there must exist additional polynomial constraints on essential matrix~$E$. Using the implicitization algorithm~\cite{CLS}, we found that the entries of $E$ additionally satisfy $7$ cubic equations. We provided them in the supplementary material.
%For example, one of these equations has the form
%\begin{multline}
%E_{22}E_{33}A_{12}\tau' - (E_{11}E_{13} + %E_{33}E_{31})A_{23}\tau'\\+ (E_{12}E_{31} + E_{21}E_{13} + %E_{32}(E_{33} - E_{11}))A_{13}\tau' + 2A_{12}A_{13}^2 = 0,
%\end{multline}
%where $\tau' = \tau + 1$, $A = E - E^\top$, and $E_{ij}$ (resp. $A_{ij}$) are the entries of matrix $E$ (resp. $A$).
The above polynomial system complemented with the new $7$ cubics has $12$ solutions (NullEx in Table~\ref{tab:template-size}). However, due to the lack of geometric interpretability for the additional cubics, in this paper we use a hand-crafted solver with slightly larger template matrix.
\end{remark}

\section{Minimal Relative Pose Solvers with $\SE(3)$ Constraints}
\label{sec:solvers}
The goodness of different solver formulations is reflected by the size of the matrix template for Gr\"obner basis computation since it directly affects both the speed and numerical accuracy of a minimal solver. Different formulated solvers are reported in Table~\ref{tab:template-size}. We compared our proposed formulation with~\cite{KBP08}, the most widely used generator the past years, and~\cite{li2020gaps}, a wrapper of a newer generator \cite{larsson2017efficient}.

\subsection{5P, 4P-RA, 5P-ST1 and 4P-RA-ST1}
NullE is the most widely used polynomial formulation for 5P, with a template matrix of $10\times20$. Template matrix of 4P-RA was recently reduced from $270\times290$ to $16\times36$ using SIR2-$u$ \cite{MartLi19}. 5P-ST1 can be solved by a 5P solver and multiply the unit translation solution $t$ by $\frac{\delta}{r^\top t}$. 4P-RA-ST1 can be solved by a 4P-RA solver in the same way.

%\subsection{4P-RA}

%For problem 4P-RA, the SIR2+$u$ formulation is preferable since leads to the smallest matrix template of size $16\times 36$. This solver has been recently introduced in~\cite{MartLi19}.

\subsection{4P-ST0}

The NullE formulation for 4P-ST0 produces the smallest template of size $10\times 20$. Note that in NullE of 4P-ST0, Eq.~\eqref{eq:e-trace} replaces an epipolar constraint of 5P and they are both linear on $E$. Therefore, 4P-ST0 can be simply solved by a NullE 5P solver by replacing the coefficients of one epipolar constraint.

\subsection{3P-RA-ST0}
\label{sec:3prast0}

For problem 3P-RA-ST0, the SIR2+$u$ formulation is preferable as it leads to the smallest $13\times 25$ matrix template. The algorithm is summarized as follows.

Three image point pairs are first used to form a $2\times 2$ matrix $F_{123}$, see Subsect.~\ref{sec:t2d}. We set $u^\top = \begin{bmatrix} \alpha & \beta & \gamma\end{bmatrix}$. Then our system consists of the following polynomial equations:
\begin{itemize}
\item $10$ equations of $m\cdot \mathsf{f}^\sigma = 0$ for $m$ being every monomial with degree up to~$2$;
\item $1$ equation $\det F_{123} = 0$;
\item $12$ equations of $m\cdot \det F_{ij}' = 0$, with $i \neq j$, $m \in \{\alpha, \beta, \gamma, 1\}$ and
$$
    F_{ij}' = \begin{bmatrix} \multicolumn{2}{c}{F_{ij}} \\ - u^\top q_i' & u^\top q_i''\end{bmatrix}.
$$
\end{itemize}
In matrix form the system can be written as $A x = 0_{23\times 1}$,
where $A$ is the $23\times 35$ coefficient matrix whose $i$-th row consists of coefficients of the $i$-th polynomial, $x$ is a monomial vector. Matrix $A$ is exactly the template produced by the Automatic Generator, see Table~\ref{tab:template-size}. However, the template's size can be further reduced if we take into account the special structure of matrix~$A$. Namely, if the first $10$ monomials in $x$ are
\begin{equation}
\label{eq:first10mons}
\alpha^4, \ \alpha^3\beta, \ \alpha^2\beta^2, \ \alpha^3\gamma, \ \alpha^2\beta\gamma, \ \alpha^2\gamma^2, \ \alpha^3, \ \alpha^2\beta, \ \alpha^2\gamma, \ \alpha^2,
\end{equation}
then matrix $A$ has the following block form $A = \begin{bmatrix}U & V \\ W & X\end{bmatrix},
\label{eq:matrixA}$
where $U$ is an upper-triangular $10\times 10$ matrix with $1$'s on its main diagonal. We conclude that matrix $A$ is equivalent to $\begin{bmatrix}U & V \\ 0_{13\times 10} & B\end{bmatrix},
\label{eq:newmatrix}$
where matrix $B = X - WU^{-1}V$ is our final template of size $13\times 25$. Matrix $B$ contains all necessary data for deriving solutions either by constructing an action matrix or by forming the $12$-th degree univariate polynomial in accordance with the hidden variable method. We provided more details on the 3P-RA-ST0 solver in the supplementary material. Readers can refer to \cite{Faugere07,MartLi19} for more usage of the above simplification.

\subsection*{Degeneration Handling}

Condition~\eqref{eq:screw-def} becomes degenerate when the rotation matrix is close to~$I$. In this case the rotation axis $r$ is ill-posed and vector $u$ becomes arbitrarily small. Enforcing condition~\eqref{eq:f0} in this case might lead to a large deviation in the direction of translation. Nevertheless, this degenerate case can be easily covered by fitting relative pose to a translation-only motion (2P-TO). The skew-symmetric essential matrix $[t]_\times$ can be easily estimated from two image feature pairs. In this paper, we estimate 4P-ST0 and 3P-RA-ST0 together with 2P-TO and accept the results with more inliers.

\section{Experiments}

\subsection{Implementation Details}

All algorithms compared in experiments are implemented by C++. The hidden variable method is used to derive solutions of polynomial systems. Roots of univariate polynomials are found using Sturm sequences. We implement 4P-RA \cite{MartLi19}, 4P-ST0, and 3P-RA-ST0\footnote{
% Source code link will be revealed after the review procedure.
Source codes are available at \texttt{http://github.com/prclibo/relative_pose}
}.
The C++ 5P solver from~\cite{Hartley2012} is used, which is regarded as the state-of-the-art fast implementation.
% We use the standard implementation from~\cite{Hartley2012} instead of its optimized version which we found to be numerically less stable.
Runtime statistics on an i5-4288U is listed in Table~\ref{tab:runtime}.

\subsection{Synthetic Data}
\label{sec:synthetic}

\begin{SCfigure}[1.1][t]
\centering
\begin{tabular}{cc}
\includegraphics[width=0.25\textwidth]{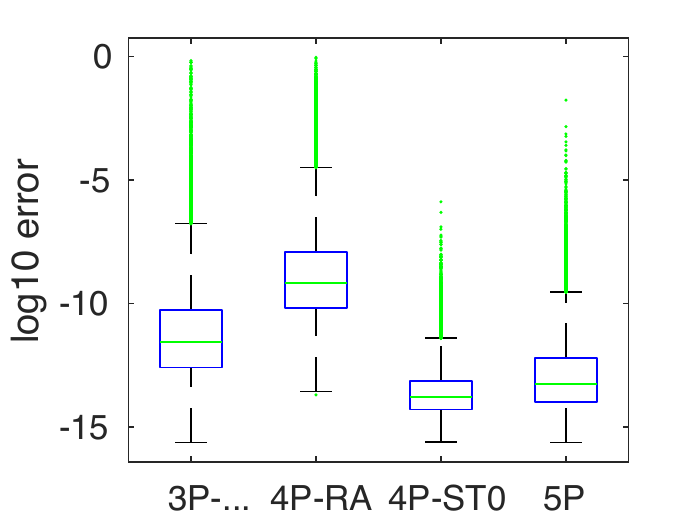} &
\includegraphics[width=0.25\textwidth]{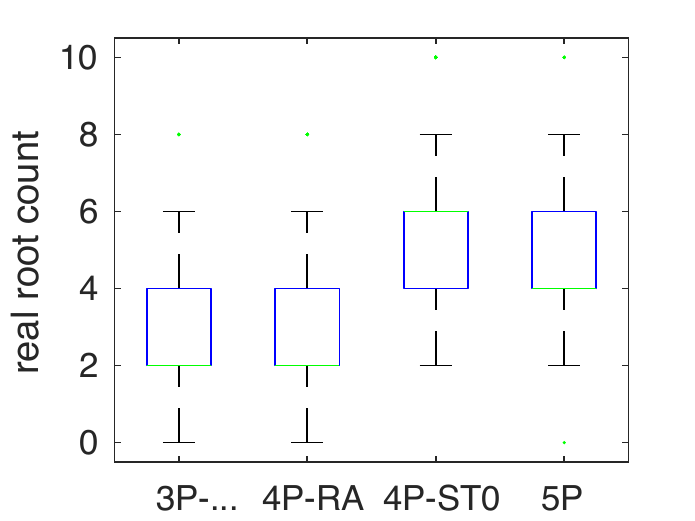} \\
(a) & (b) \\
\end{tabular}
\caption{(a) Numerical accuracy comparison of the solvers; (b) Statistics on the number of real roots for each solver}
\label{fig:numerical}
\end{SCfigure}

\begin{SCtable}[1][t]
\footnotesize
\centering
\begin{tabular}{lrrrr}
\hline
Minimal Solver & 3P-RA-ST0 & 4P-RA & 4P-ST0 & 5P \\ \hline
Average Time & 28 $\mu$s & 34 $\mu$s & 26 $\mu$s & 25 $\mu$s \\
\hline
\end{tabular}\smallskip
\caption{Average runtime comparison of the solvers}
\label{tab:runtime}
\end{SCtable}

\begin{figure*}[ht]
\centering
\begin{tabular}{cccc}
\includegraphics[width=0.25\textwidth]{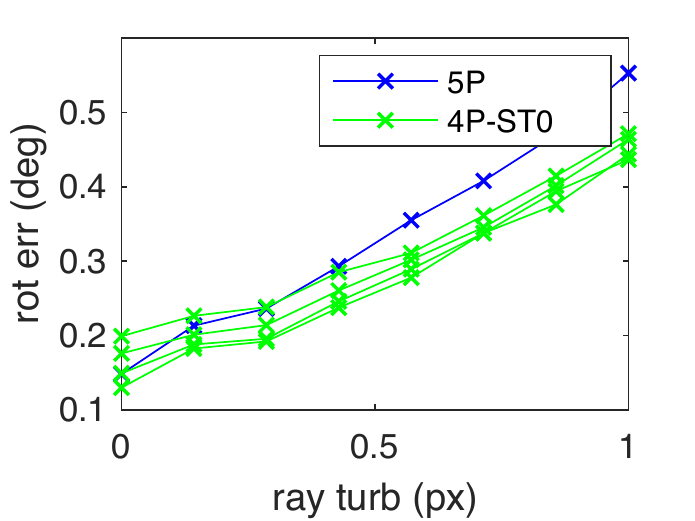}&
\includegraphics[width=0.25\textwidth]{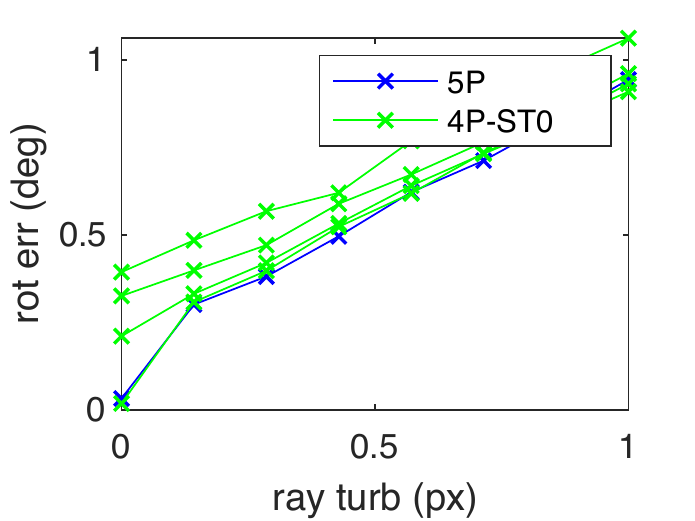}&
\includegraphics[width=0.25\textwidth]{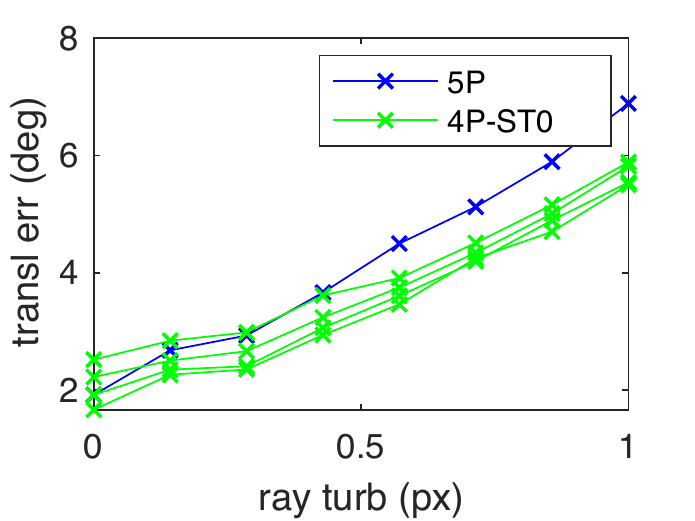}&
\includegraphics[width=0.25\textwidth]{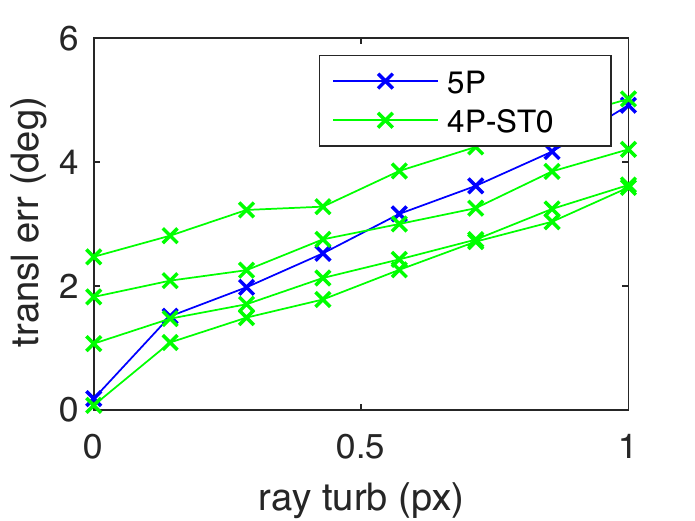}\\
(a) forward motion & (b) sideway motion & (c) forward motion & (d) sideway motion
\end{tabular}
\caption{Estimation error plot of 4P-ST0 and 5P on synthetic data: (a, b) rotation errors; (c, d) translation errors}
\label{fig:4pst0-ransac}
\end{figure*}

Synthetic data are generated to illustrate the algorithm performance. Synthetic image features are generated from a $60^\circ$ field of view with focal length in $500$px. We test algorithm performance under Gaussian image noise whose std ranges in $0\mbox{--}1$px. Synthetic data is generated for forward motion and sideway motion. Rotation angle of a pose pair is randomly generated from Gaussian with std of $5^\circ$. The rotation angle measurement is disturbed by Gaussian noise (derived from the widely used Brownian process model for IMU noise) with std ranging in $0\mbox{--}1^\circ$. To test the performance of 4P-ST0 and 3P-RA-ST0 under non-perfectly planar motion, we first produces unit translation with zero component on rotation axis. Then the translation is disturbed along the rotation axis with Gaussian noise whose std ranges in $0\mbox{--}5\%$.

\begin{figure*}[ht]
\centering
\begin{tabular}{cccc}
\includegraphics[width=0.25\textwidth]{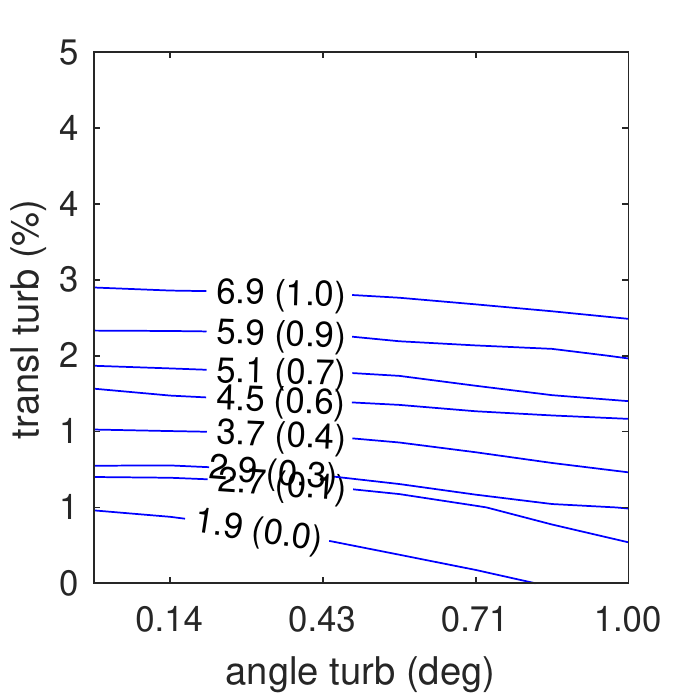} &
\includegraphics[width=0.25\textwidth]{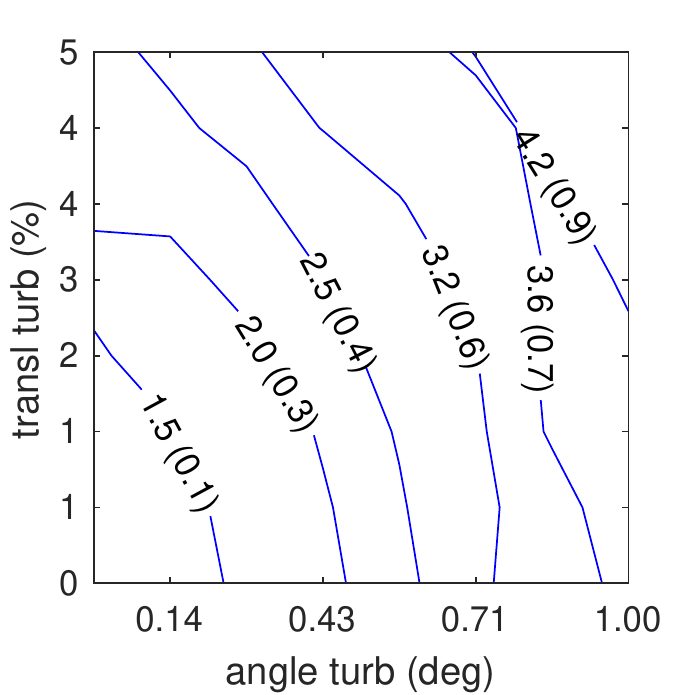} &
\includegraphics[width=0.25\textwidth]{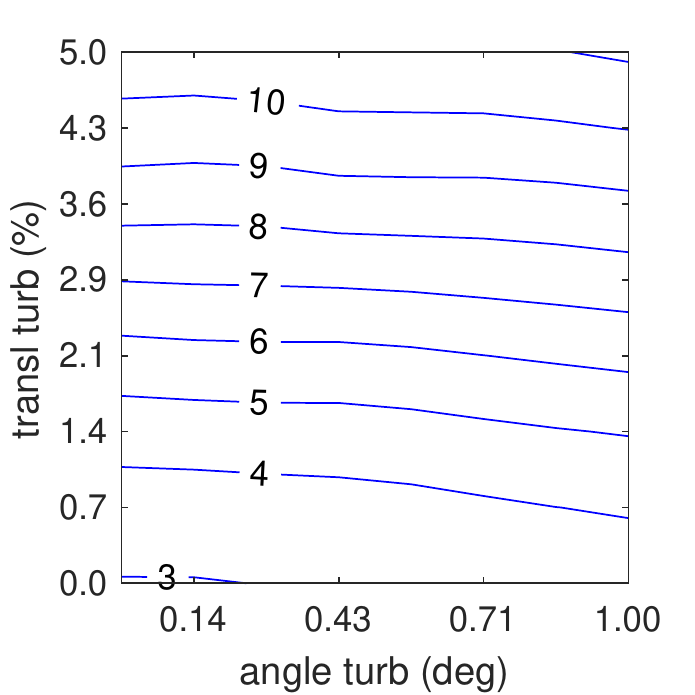} &
\includegraphics[width=0.25\textwidth]{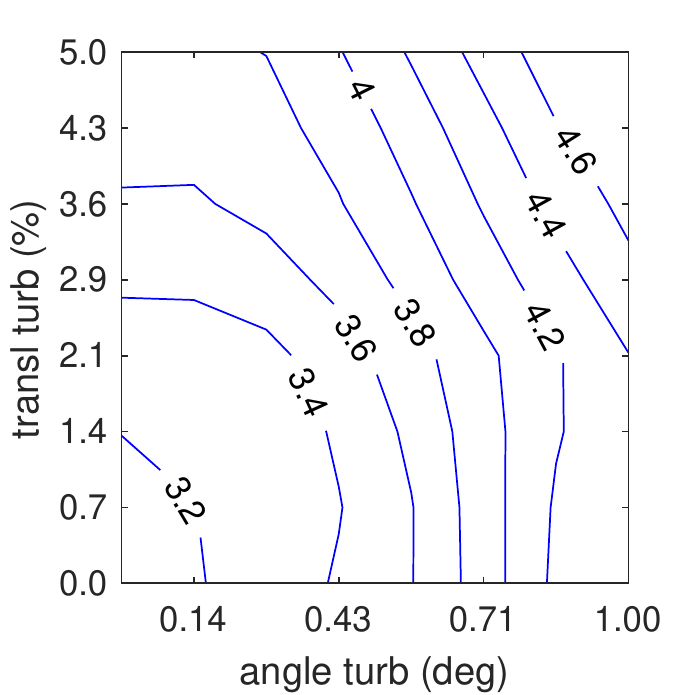} \\
(a) forward motion & (b) sideway motion & (c) forward motion & (d) sideway motion
\end{tabular}
\caption{Estimation error comparison of 3P-RA-ST0 and 5P: (a, b) contour curves $\xi(\epsilon)$ denote that under image feature noise std $\epsilon$ px, when the error of rotation angle and screw translation is at the bottom left of the curve, 3P-RA-ST0 outperforms 5P with translation error no more than~$\xi$; (c, d) translation error contour curve over different rotation angle and screw translation disturbance, with image feature noise std fixed as $1$px}
\label{fig:3prast0-ransac}
\end{figure*}

The numerical accuracy of each algorithm is compared and listed in Fig.~\ref{fig:numerical}(a). The numerical error is measured by the value $\min\limits_i\|R_i - \bar R\|$, where $i$ counts all real solutions and $\bar R$ is the ground truth relative rotation matrix. The number of real roots is also counted for each algorithm and listed in Fig.~\ref{fig:numerical}(b). We observe that 4P-ST0 in NullE formulation generally has more real roots compared to 5P. The number of real roots also affects the computational efficiency in some RANSAC frameworks like OpenCV where each real solution must be verified by computing the reprojection error over all image feature pairs.

In the experiments on both synthetic data and real data, the error of rotation is measured by the rotation angle between the estimated and groundtruth rotation. The error of translation is measure by the angle between the unit groundtruth translation and the estimated translation. Forward motion and sideway motion are experimented separately. The mean estimation error of 4P-ST0 and 5P against image ray disturb is shown in Fig.~\ref{fig:4pst0-ransac}. The estimation is executed on $100$ image feature pairs with $30\%$ outliers under RANSAC. Green curves from bottom to top represent 4P-ST0 estimation with different screw translation disturbance along rotation axis $\{0\%, 1.66\%, 3.33\%, 5\%\}$. As it is mentioned in many previous works, the rotation error is generally small for different solvers. Regarding translation error, the advantage of 4P-ST0 over 5P is more significant for forward motion, which is considered as the more common and difficult case than sideway motion.

The performance of 3P-RA-ST0 is affected by rotation angle error, screw translation error and also image feature error. To simplify our visualization, in Fig.~\ref{fig:3prast0-ransac}(a) and Fig.~\ref{fig:3prast0-ransac}(b), we fix the image feature error at different level and consider the translation error surface against the rotation angle and screw translation errors. The intersection contour of this surface with the error surface of 5P is plotted. The bottom-left area of each curve denotes the error level under which 3P-RA-ST0 can outperform 5P. We only compare the translation error here as the rotation error is generally similar for different methods. Compared to 4P-ST0, we find that 3P-RA-ST0 is more sensitive to screw translation error in forward motion. For example, with perfect rotation angle and image feature noise std as $1$px in forward motion, 3P-RA-ST0 outperforms 5P only when screw translation error is less than $3\%$ (Fig.~\ref{fig:3prast0-ransac}(a)), while 4P-ST0 outperforms 5P even when screw translation error is $5\%$ (Fig.~\ref{fig:4pst0-ransac}(c)).

\subsection{Real-World Data}

We compare our approaches on multiple datasets collected on indoor mobile robots or outdoor autonomous vehicles, which are two popular modern robot applications with planar motion. Experimented datasets include:

\begin{itemize}
\item RawSeeds-Bicocca~\cite{Bonarini06}: Indoor mobile robot data, with IMU and odometer available for rotation angle. Front camera (FC) images are used.

\item TUM-RGBD-SLAM~\cite{sturm12iros}, Robot@Home~\cite{Ruiz17}: Indoor mobile robot data, without available angle measurement. Images from the front RGBD camera (FC) are used. The left RGBD camera (LC) of Robot@Home is also experimented.

\item KITTI~\cite{Geiger12}, UMich~\cite{pandey11}, RobotCar~\cite{Maddern17}: Autonomous vehicle data, with fused GPS/INS data available. Front camera (FC) images are used. The left camera (LC) of UMich is also used. Rotation angle from the fused GPS/INS pose is used.
\end{itemize}

Consecutive image pairs with translational movement larger than $0.1$m in indoor data and $1$m in outdoor data are used in experiments. Performance comparison on indoor data is shown in Table~\ref{tab:indoor}. It is seen that 4P-ST0 outperforms 5P in almost all cases, which is consistent with the synthetic data results. With IMU data on RawSeeds, 3P-RA-ST0 further improves the estimation. Results with odometry angle have no improvement, implying the accuracy of odometry angle might be low in this dataset.

On autonomous driving data in Table~\ref{tab:outdoor} however, the performance varies. We observe that for environments like broad road or highway, 4P-ST0 outperforms 5P. For environments like urban narrow road, 5P has the better accuracy. This environment difference corresponds to different screw translation disturbance on a planar motion assumption, as roads are less planar and vehicles might tilt more on urban road. With some more analysis, we found that a portion of relative poses have screw translation of more than $20\%$ ($r^\top t / \|t\| > 0.2$) in urban autonomous driving scenarios, which explains the poor performance of 4P-ST0.

We also note that 3P-RA-ST0 performs better on UMich left camera than 4P-ST0, corresponding to the observation in Fig.~\ref{fig:3prast0-ransac} that 3P-RA-ST0 is less sensitive to screw translation error for sideway motion.
% Besides, 3P-RA-ST0 performs poorly for RobotCar and some sequences in KITTI. We found that this is due to two reasons: 1) 3P-RA-ST0 is more sensitive to screw translation disturbance than 4P-ST0 in forward motion as observed in synthetic experiments; 2) some autonomous driving datasets have very high resolution images, which relatively increases the precision of image feature positions.

\section{Conclusions}

In this paper we show that known $\SE(3)$ invariants can be used to constrain the minimal relative pose estimation problem. Compared to existing relative pose problems with contraints, the proposed methods are more flexible and convenient since extrinsics are not required to transform $\SE(3)$ invariant to the camera frame.

We also comprehensively revise and relate to each other existing formulations of the relative pose problem. The discovered relationship provides a deeper understanding to these previous methods. This knowledge help formulate the most efficient solvers for the proposed relative pose problem with $\SE(3)$ constraints.

A series of experiments on synthetic and real datasets show practicality of the proposed solvers in robotic perception especially for indoor robots.

\begin{table}[!h]
% Minipage top aligned table:
% https://texblog.org/2007/08/01/placing-figurestables-side-by-side-minipage/#comment-4322
\begin{minipage}[t]{0.55\textwidth}
	\centering\footnotesize
	\begin{tabular}[t]{crrrrrrr}
		\hline
		\multirow{2}{*}{Dataset} & \multirow{2}{*}{\%} & \multicolumn{2}{c}{3P-RA-ST0} & \multicolumn{2}{c}{4P-RA} & \multirow{2}{*}{4P-ST0} & \multirow{2}{*}{5P}\\
		\cline{3-6}
		&& Odo & IMU & Odo & IMU & &\\ \hline
		\multirow{3}{*}{\begin{tabular}{@{}c@{}}RSeeds \\ 0226b\end{tabular}}
		& 25 & 8.9 & \textbf{5.6} & 10.3 & {6.7} & 8.8 & 10.8 \\
		& 50 & 16.0 & \textbf{9.4} & 16.0 & {10.5} & 16.1 & 17.5 \\
		& 75 & 26.9 & \textbf{15.4} & 29.1 & { 15.7} & 29.1 & 27.7 \\ \hline
		\multirow{3}{*}{\begin{tabular}{@{}c@{}}RSeeds \\ 0226a\end{tabular}}
		& 25 & 7.3 & \textbf{4.5} & 10.5 & 6.7 & 7.0 & 9.1 \\
		& 50 & 14.6 & \textbf{7.7} & 18.7 & 11.0 & 12.5 & 15.1 \\
		& 75 & 27.4 & \textbf{11.5} & 33.7 & 15.7 & 21.8 & 23.6 \\ \hline
		\multirow{3}{*}{\begin{tabular}{@{}c@{}}RSeeds \\ 0225b\end{tabular}}
		& 25 & 10.1 & \textbf{5.5} & 10.7 & 7.5 & 8.8 & 10.6 \\
		& 50 & 17.4 & \textbf{9.2} & 17.1 & 11.2 & 16.4 & 17.8 \\
		& 75 & 35.9 & \textbf{16.2} & 34.7 & 16.8 & 30.8 & 30.0 \\ \hline
		\multirow{3}{*}{\begin{tabular}{@{}c@{}}RSeeds \\ 0225a\end{tabular}}
		& 25 & 9.1 & \textbf{4.5} & 11.5 & 6.6 & 6.9 & 9.8 \\
		& 50 & 17.8 & \textbf{8.3} & 20.4 & 11.8 & 13.3 & 15.8 \\
		& 75 & 32.6 & \textbf{14.0} & 39.3 & 17.4 & 22.8 & 24.9 \\ \hline
	\end{tabular}
	\\\vspace{6pt}
        \centering\footnotesize
	\begin{tabular}{crrr}
		\hline
		Dataset & \% & {4P-ST0} & {5P}\\ \hline
		\multirow{3}{*}{\begin{tabular}{@{}c@{}}TUM \\ RGBD\\ \#360\end{tabular}}
		& 25 & \textbf{3.7} & 4.1\\
		& 50 & \textbf{6.4} & 6.6\\
		& 75 & \textbf{12.1} & 14.0\\ \hline
		\multirow{3}{*}{\begin{tabular}{@{}c@{}}TUM \\ RGBD\\ \#2\end{tabular}}
		& 25 & \textbf{2.9} & 4.3 \\
		& 50 & \textbf{5.7} & 7.7 \\
		& 75 & 16.0 & \textbf{15.2} \\ \hline
		\multirow{3}{*}{\begin{tabular}{@{}c@{}}R@H \\ anto\\ s1 FC\end{tabular}}
		& 25 & \textbf{5.1} & 7.9 \\
		& 50 & \textbf{10.3} & 14.3 \\
		& 75 & \textbf{15.7} & 22.8 \\ \hline
		\multirow{3}{*}{\begin{tabular}{@{}c@{}}R@H \\ alma\\ s1 FC\end{tabular}}
		& 25 & \textbf{3.4} & 7.2 \\
		& 50 & \textbf{9.3} & 12.0 \\
		& 75 & \textbf{16.3} & 16.9 \\ \hline
		\multirow{3}{*}{\begin{tabular}{@{}c@{}}R@H \\ pare\\ s1 FC\end{tabular}}
		& 25 & \textbf{4.9} & 5.3 \\
		& 50 & \textbf{8.5} & 9.9 \\
		& 75 & \textbf{16.1} & 23.5 \\ \hline
		\multirow{3}{*}{\begin{tabular}{@{}c@{}}R@H \\ sarmis\\ s1 FC\end{tabular}}
		& 25 & \textbf{7.3} & 9.0 \\
		& 50 & \textbf{12.2} & 17.4 \\
		& 75 & \textbf{21.4} & 28.0 \\ \hline
	\end{tabular}~
        \begin{tabular}{crrr}
		\hline
		Dataset & \% & {4P-ST0} & {5P}\\ \hline
		\multirow{3}{*}{\begin{tabular}{@{}c@{}}TUM \\ RGBD\\ \#1\end{tabular}}
		& 25 & \textbf{2.7} & 3.8 \\
		& 50 & \textbf{6.2} & 7.2 \\
		& 75 & \textbf{14.9} & 15.2 \\ \hline
		\multirow{3}{*}{\begin{tabular}{@{}c@{}}TUM \\ RGBD\\ \#3\end{tabular}}
		& 25 & \textbf{2.7} & 3.1 \\
		& 50 & \textbf{5.9} & 6.2 \\
		& 75 & {13.2} & \textbf{12.3} \\ \hline
		\multirow{3}{*}{\begin{tabular}{@{}c@{}}R@H \\ anto\\ s1 LC\end{tabular}}
		& 25 & \textbf{8.7} & 9.2 \\
		& 50 & \textbf{14.3} & 21.5 \\
		& 75 & \textbf{22.9} & 39.3 \\ \hline
		\multirow{3}{*}{\begin{tabular}{@{}c@{}}R@H \\ alma\\ s1 LC\end{tabular}}
		& 25 & \textbf{8.1} & 8.5 \\
		& 50 & \textbf{13.3} & 15.4 \\
		& 75 & \textbf{22.4} & 24.8 \\ \hline
		\multirow{3}{*}{\begin{tabular}{@{}c@{}}R@H \\ pare\\ s1 LC\end{tabular}}
		& 25 & \textbf{5.0} & 6.4 \\
		& 50 & \textbf{8.6} & 11.7 \\
		& 75 & \textbf{13.7} & 21.5 \\ \hline
		\multirow{3}{*}{\begin{tabular}{@{}c@{}}R@H \\ sarmis\\ s1 LC\end{tabular}}
		& 25 & \textbf{9.4} & 12.9 \\
		& 50 & \textbf{17.1} & 29.8 \\
		& 75 & \textbf{34.7} & 46.9 \\ \hline
	\end{tabular}
	\\\vspace{5pt}
	\caption{Translation error angle (deg) quantiles on indoor real data. Adding $\SE(3)$ invariant measurements improves relative pose estimation in most cases}
	\label{tab:indoor}
\end{minipage}
\hspace{5pt}
\begin{minipage}[t]{0.43\textwidth}
	\centering\footnotesize
	\begin{tabular}[t]{crrrrr}
		\hline\rule{0pt}{12pt}
		Dataset & \% & 3P-... & 4P-RA & 4P-ST0 & 5P\\[2pt] \hline
		\multirow{3}{*}{\begin{tabular}{@{}c@{}}~UMich~ \\ \#1 FC\end{tabular}}
		& 25 & 2.3 & 3.1 & \textbf{2.0} & 2.2\\
		& 50 & 5.1 & 7.7 & \textbf{4.2} & 4.6\\
		& 75 & 14.2 & 37.0 & \textbf{9.6} & 9.9\\ \hline
		\multirow{3}{*}{\begin{tabular}{@{}c@{}}UMich \\ \#2 FC\end{tabular}}
		& 25 & 1.7 & 1.7 & \textbf{1.5} & 1.6\\
		& 50 & 3.0 & 3.0 & \textbf{2.8} & 2.9\\
		& 75 & 5.0 & 5.1 & 4.8 & \textbf{4.8} \\ \hline
		\multirow{3}{*}{\begin{tabular}{@{}c@{}}UMich \\ \#1 LC\end{tabular}}
		& 25 & 1.4 & 1.5 & \textbf{1.4} & 1.6 \\
		& 50 & \textbf{2.5} & 2.7 & 2.6 & 2.9 \\
		& 75 & \textbf{5.3} & 5.1 & 5.7 & 5.6 \\ \hline
		\multirow{3}{*}{\begin{tabular}{@{}c@{}}UMich \\ \#2 LC\end{tabular}}
		& 25 & 2.1 & 2.2 & \textbf{2.1} & 2.4 \\
		& 50 & \textbf{4.1} & 4.7 & 4.1 & 4.9 \\
		& 75 & \textbf{8.2} & 10.8 & 8.5 & 10.1 \\ \hline
		\multirow{3}{*}{\begin{tabular}{@{}c@{}} KITTI \\ \#1\end{tabular}}
		& 25 & 1.2 & 1.3 & 1.0 & \textbf{0.9} \\
		& 50 & 2.3 & 3.0 & 1.8 & \textbf{1.7} \\
		& 75 & 4.8 & 8.7 & 3.2 & \textbf{3.1} \\ \hline
	\end{tabular}
	\\\vspace{4pt}
	\begin{tabular}{crrrrr}
		\hline\rule{0pt}{12pt}
		Dataset & \% & 3P-... & 4P-RA & 4P-ST0 & 5P\\[2pt] \hline
		\multirow{3}{*}{\begin{tabular}{@{}c@{}}KITTI \\ \#4\end{tabular}}
		& 25 & 0.9 & 0.9 & \textbf{0.8} & 0.9 \\
		& 50 & 1.4 & 1.6 & \textbf{1.2} & 1.6 \\
		& 75 & 2.4 & 2.9 & \textbf{2.1} & 2.6 \\ \hline
		\multirow{3}{*}{\begin{tabular}{@{}c@{}}KITTI \\ \#6\end{tabular}}
		& 25 & 0.9 & 1.2 & \textbf{0.8} & 1.0 \\
		& 50 & 1.7 & 2.1 & \textbf{1.4} & 1.8 \\
		& 75 & 3.0 & 4.4 & \textbf{2.5} & 3.0 \\ \hline
		\multirow{3}{*}{\begin{tabular}{@{}c@{}}KITTI \\ \#9\end{tabular}}
		& 25 & 1.7 & 2.1 & \textbf{1.1} & 1.1 \\
		& 50 & 3.5 & 4.4 & 2.0 & \textbf{2.0} \\
		& 75 & 8.1 & 12.1 & 3.7 & \textbf{3.3} \\ \hline
		\multirow{3}{*}{\begin{tabular}{@{}c@{}} RobotCar\\ 05-14\\13:46 \end{tabular}}
		& 25 & 4.5 & 3.3 & 2.4 & \textbf{2.2} \\
		& 50 & 9.3 & 6.4 & 3.4 & \textbf{2.8} \\
		& 75 & 26.0 & 20.9 & 5.5 & \textbf{3.6} \\ \hline
		\multirow{3}{*}{\begin{tabular}{@{}c@{}} RobotCar\\ 06-26\\08:53 \end{tabular}}
		& 25 & 3.7 & 2.5 & 2.7 & \textbf{2.3} \\
		& 50 & 5.6 & 3.6 & 3.9 & \textbf{2.9} \\
		& 75 & 10.1 & 5.6 & 6.8 & \textbf{3.7} \\ \hline
	\end{tabular}
	\\\vspace{5pt}
	\caption{Translation error angle (deg) quantiles on outdoor autonomous driving real data}
	\label{tab:outdoor}
\end{minipage}
\end{table}

\medskip
\noindent\small{\textbf{Acknowledgements.} The work of E.M. was supported by Act 211 Government of the Russian Federation, contract No.~02.A03.21.0011.}

\clearpage
\bibliographystyle{splncs04}
\bibliography{se3}
\clearpage

%\appendix
%\section{Title of Appendix A}

\include{se3_sup}

\end{document}

%% file: se3_sup.tex
\setcounter{section}{0}
\setcounter{theorem}{0}
\title{Relative Pose Estimation of Calibrated Cameras with Known $\SE(3)$ Invariants}
\subtitle{\large Supplementary Material}

\author{Bo Li \and Evgeniy	Martyushev \and Gim Hee Lee}
\institute{}

\maketitle

\setcounter{equation}{22}
\setcounter{figure}{4}

\section{Seven Cubics from 3P-RA-ST0 under NullE (NullEx)}

\begin{theorem}
Let $E = [t]_\times R$ be an essential matrix with $R$ being a rotation around a vector $u$. If $u^\top t = 0$, then the following equations hold:
\begin{multline}
\label{eq:7cubics1}
(A_{ki}E_{ii}E_{jk} + A_{ij}E_{kk}E_{ii} - A_{jk}(E_{kk}E_{ik} - E_{ki}E_{jj} + E_{kj}E_{ji} + E_{ij}E_{jk}))\tau' \\+ 2A_{ij}A_{jk}^2 = 0,
\end{multline}
\begin{multline}
\label{eq:7cubics2}
(A_{12}A_{23}E_{31} + A_{12}A_{31}E_{32} + A_{12}E_{12}E_{33} + A_{23}A_{31}E_{21} + A_{23}E_{11}E_{23} + A_{31}E_{22}E_{31})\tau' \\+ 2A_{12}A_{23}A_{31} = 0,
\end{multline}
where $\tau' = \tr R + 1$, $A = E - E^\top$, and $E_{ij}$ (resp. $A_{ij}$) are the entries of matrix $E$ (resp. $A$). The indices $i$, $j$, $k$ are intended to be different in Eq.~\eqref{eq:7cubics1}.
\end{theorem}

\begin{proof}
By a straightforward computation. Constraints~\eqref{eq:7cubics1} and~\eqref{eq:7cubics2} are derived by using the implicitization algorithm~\cite{CLS}. First, assuming that $R$ is represented by formula~\eqref{eq:quat-to_rot}, we constructed the polynomial ideal $J$ generated by $E_{ij} - ([t]_\times R)_{ij}$, $u^\top t$ and $\|u\|^2 + \sigma^2 - 1$. Then we computed the Gr\"obner basis of $J$ with respect to a lexicographic ordering where the entries of vectors $t$ and $u$ are greater than the entries of matrix $E$ and scalar $\sigma$. Thus we got the elements of the Gr\"obner basis not involving vectors $t$ and $u$, i.e. the basis of the elimination ideal $J\cap \mathbb C[E_{11}, \ldots, E_{33}, \sigma]$. This basis contains all polynomials from~\eqref{eq:det-e}~--~\eqref{eq:e-trace} as well as seven more elements represented by~\eqref{eq:7cubics1} and~\eqref{eq:7cubics2}.
\end{proof}

\section{3P-RA-ST0: Solver Details}

Recall that in Subsect.~\ref{sec:3prast0} we reduced solving the 3P-RA-ST0 problem to finding real roots of the polynomial system $Ax = 0_{23\times 1}$, where $A$ is a coefficient matrix of size $23\times 35$ and $x$ is a monomial vector. The structure of matrix $A$ is shown in Fig.~\ref{fig:mat_3prast0}(a). It can be seen that $A$ has the following block form
\[
A = \begin{bmatrix}U & V \\ W & X\end{bmatrix},
\]
where $U$ is an upper-triangular $10\times 10$ matrix with $1$'s on its main diagonal. Then it follows that $\det U = 1$ and the inverse to $U$ always exists. By elementary row operations, matrix $A$ is equivalent to
\[
\begin{bmatrix}U & V \\ 0_{13\times 10} & B\end{bmatrix},
\]
where matrix $B = X - WU^{-1}V$ of size $13\times 25$ contains all necessary data for computing all solutions of the initial polynomial system. The structure of matrix $B$ is shown in Fig.~\ref{fig:mat_3prast0}(b).

\begin{figure}%[ht]
\centering
\begin{tabular}{cc}
\begin{tabular}{c}
\includegraphics[width=0.5\textwidth]{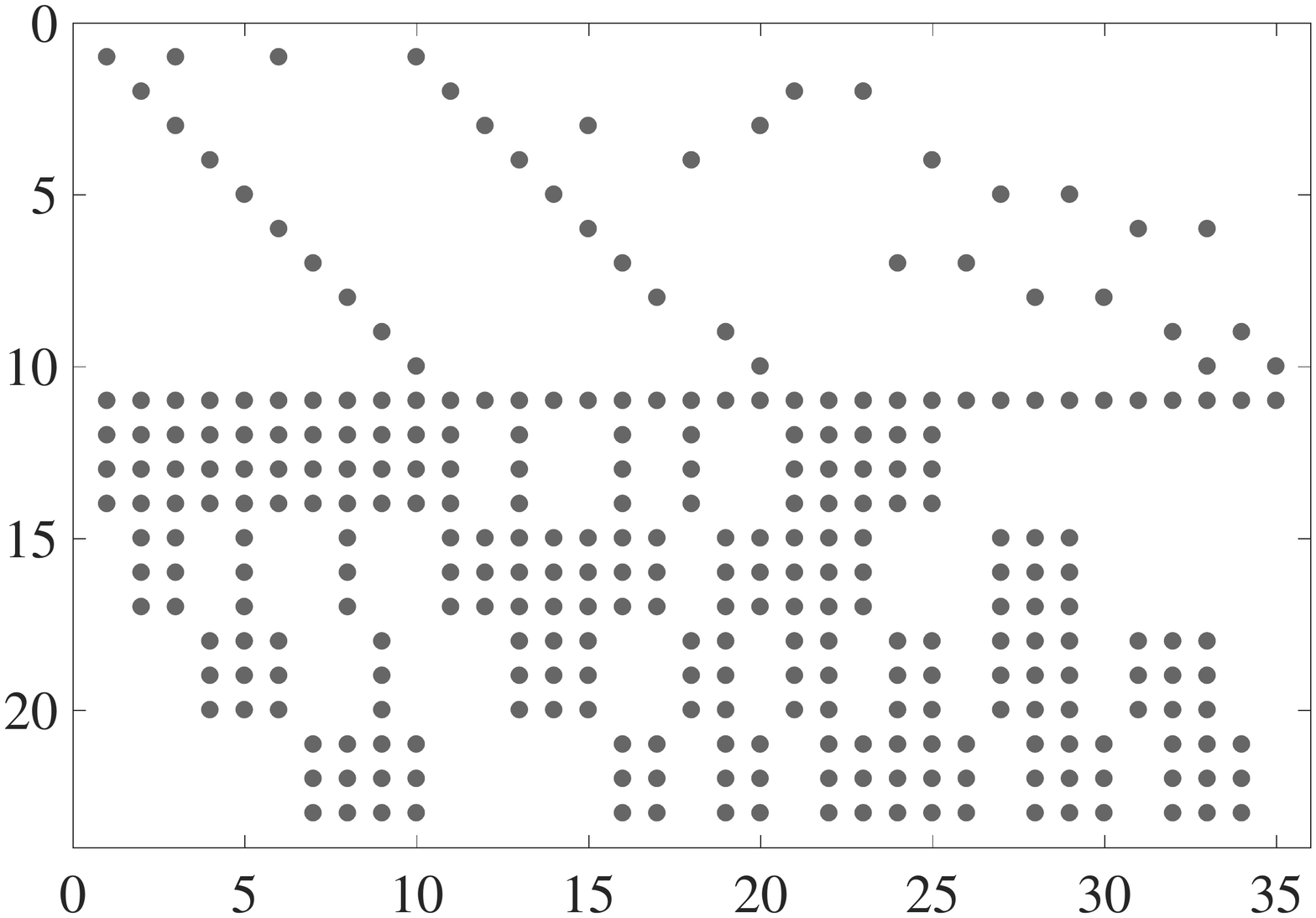} \vspace{-12pt}\\
(a)
\end{tabular} &
\begin{tabular}{c}
\includegraphics[width=0.35\textwidth]{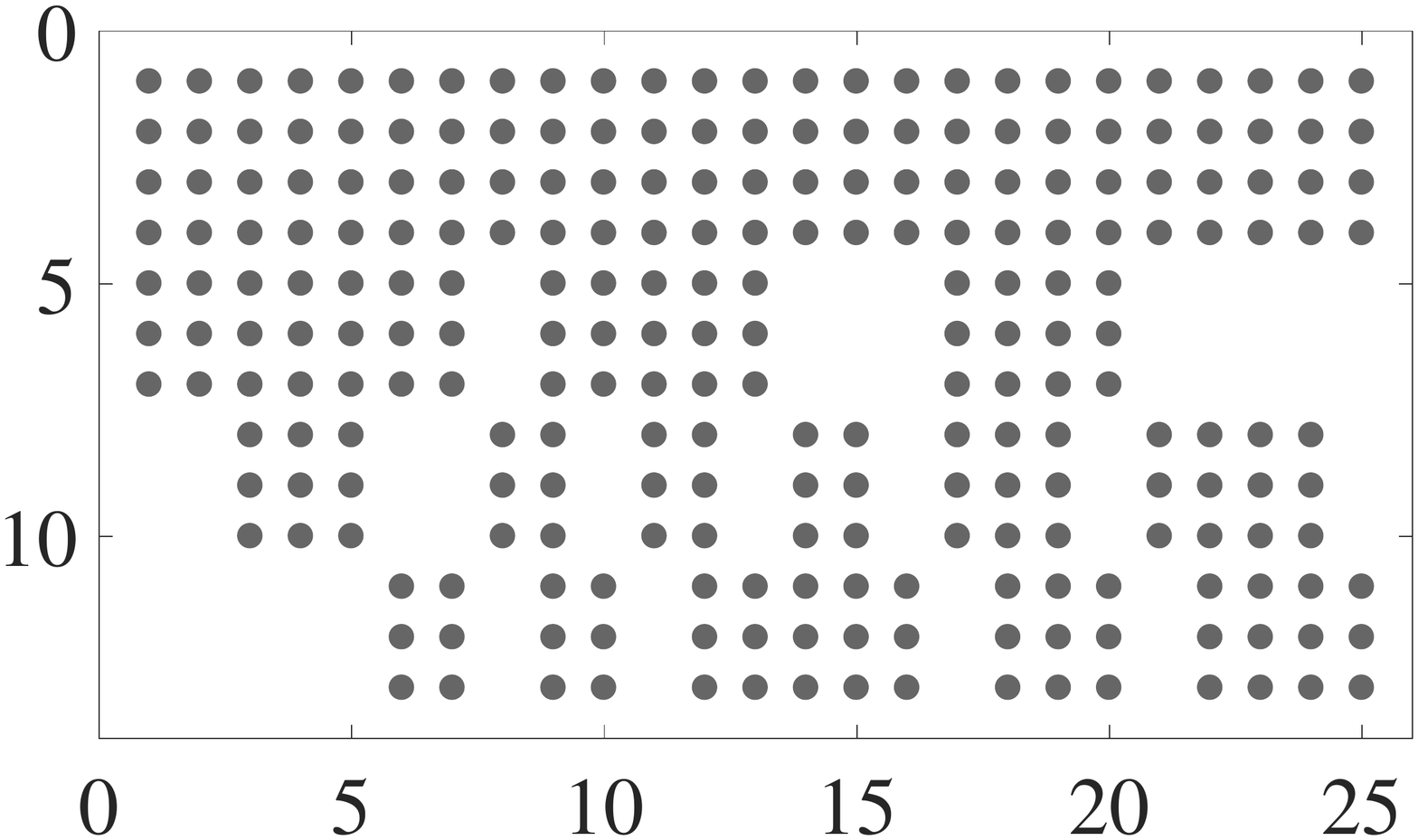} \vspace{-12pt}\\
(b) \vspace{-5pt}\\
\includegraphics[width=0.35\textwidth]{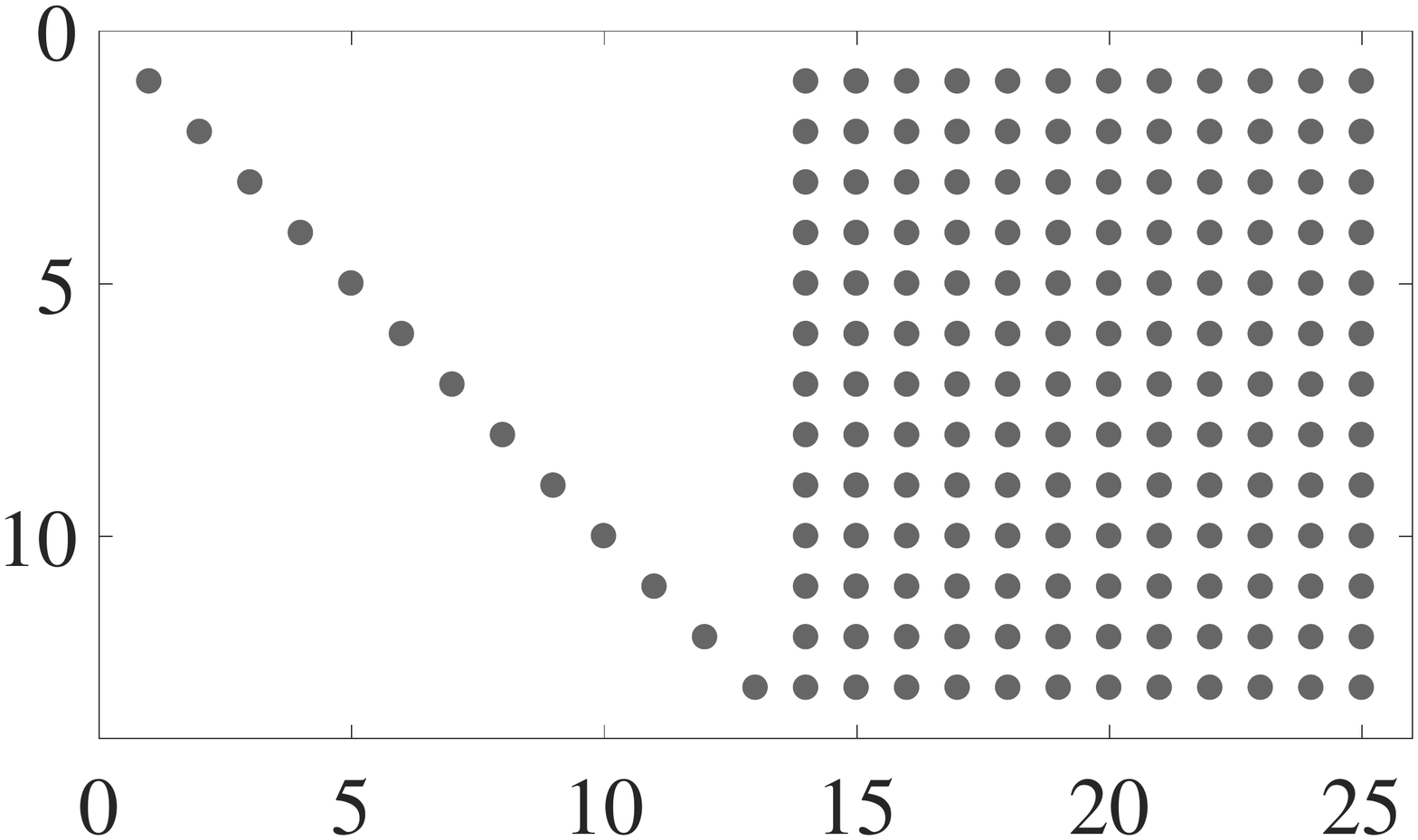} \vspace{-12pt}\\
(c)
\end{tabular}
\end{tabular}
\caption{Left: The sparse structure of $23\times 35$ matrix $A$. The gray circles represent non-zero entries. The $10\times 10$ left upper submatrix of $A$ is upper-triangular and its main diagonal consists of $1$'s. Right top: The final $13\times 25$ elimination template $B$ for computing the 12-th degree univariate polynomial. Right bottom: The reduced row echelon form of $B$}
\label{fig:mat_3prast0}
\end{figure}

It is worth mentioning that matrix $B$ should not be computed by its definition in an implementation. Instead, it is more efficient to use quite simple pre-computed formulas for the nonzero entries of $B$.

The initial polynomial system $Ax = 0_{23\times 1}$ is equivalent to the system $By = 0_{13\times 1}$, where $y$ is a new monomial vector consisting of $25$ monomials. We assume that the hidden variable is $\gamma$. Let the last $17$ monomials in $y$ be
\[
\beta^2\gamma \quad \beta^2 \quad \alpha\beta\gamma^2 \quad \alpha\beta\gamma \quad \alpha\beta \quad \alpha\gamma^2 \quad \alpha\gamma \quad \alpha \quad \beta\gamma^3 \quad \beta\gamma^2 \quad \beta\gamma \quad \beta \quad \gamma^4 \quad \gamma^3 \quad \gamma^2 \quad \gamma \quad 1.
\]
Let $\widetilde B$ be the reduced row echelon form of $B$, see Fig.~\ref{fig:mat_3prast0}(c). If $(\widetilde B)_i$ is the $i$-th row of $\widetilde B$, then we can write
\[
\begin{bmatrix}\gamma(\widetilde B)_{13} - (\widetilde B)_{12}\\ \gamma(\widetilde B)_{12} - (\widetilde B)_{11}\\ \gamma(\widetilde B)_{10} - (\widetilde B)_{9}\end{bmatrix} y = C(\gamma) \begin{bmatrix}\alpha \\ \beta \\ 1 \end{bmatrix} = 0_{3\times 1}.
\]
Here we defined the $3\times 3$ matrix
\[
C(\gamma) = \begin{bmatrix}[3] & [4] & [5]\\ [3] & [4] & [5]\\ [3] & [4] & [5]\end{bmatrix},
\]
where $[n]$ means a univariate polynomial in $\gamma$ of degree $n$. It follows that $\gamma$ is a root if and only if matrix $C(\gamma)$ is degenerate. The problem is thus converted to finding all real roots of the 12-th degree polynomial $p(\gamma) = \det C(\gamma)$. Let $\gamma_0$ be a real root of $p$. The remaining components of vector $u$, i.e. $\alpha$ and $\beta$, are computed from the right null-vector of matrix $C(\gamma_0)$. Hence we get all solutions for vector~$u$.

It is important to note that because of numerical inaccuracy, the solution for $u$ does not exactly satisfy Eq.~\eqref{eq:unit-quaternion}. In order to rectify the solution, we replace $u$ with the vector
\[
\hat u = \frac{\sqrt{1 - \sigma^2}}{\|u\|}\, u.
\]
Then rotation matrix $R$ is computed from the unit quaternion $\begin{bmatrix}\sigma & \hat u^\top\end{bmatrix}$ by formula~\eqref{eq:quat-to_rot}.

Using the rigid motion ambiguity of the world coordinate frame, we set $t' = 0_{3\times 1}$. The translation vector $t''$ is found from the epipolar constraints as the right null-vector of the matrix
\[
\begin{bmatrix}
{q'_1}^\top R^\top [q''_1]_\times\\
{q'_2}^\top R^\top [q''_2]_\times\\
{q'_3}^\top R^\top [q''_3]_\times
\end{bmatrix}.
\]
The scale ambiguity allows us to set $\|t''\| = 1$. Finally, the sign of $t''$ is disambiguated by means of the cheirality constraint~\cite{HZ,Nister}.